\documentclass[11pt]{article}
\usepackage{fullpage}
\usepackage{amsmath,amsfonts,amsthm,amssymb}
\usepackage{enumerate}
\usepackage{cite}
\usepackage{hyperref}

\newtheorem{result}{Theorem}
\newtheorem{theorem}{Theorem}[section]

\newtheorem{corollary}[theorem]{Corollary}
\newtheorem{lemma}[theorem]{Lemma}
\newtheorem{claim}[theorem]{Claim}
\newtheorem*{remark}{Remark}
\newenvironment{proofof}[1]{\noindent {\em Proof of #1.  }}{\hfill$\Box$}

\newtheorem{definition}[theorem]{Definition}

\newcommand{\card}[1] {\left\vert #1 \right\vert}
\newcommand{\set}[1] {\left\{ #1 \right\}}

\newcommand{\A}{{\cal A}}

\newcommand{\C}{{\cal C}}
\newcommand{\D}{{\cal D}}
\newcommand{\F}{{\cal F}}
\renewcommand{\H}{{\cal H}}

\renewcommand{\S}{{\cal S}}
\newcommand{\X}{{\cal X}}
\newcommand{\R}{\mathbb{R}}
\newcommand{\eps}{\varepsilon}

\newcommand{\sgn}{{\rm sgn}}

\DeclareMathOperator*{\E}{\mathbb{E}}

\DeclareMathOperator*{\argmin}{argmin}

\newcommand{\norm}[1] {\left\| #1 \right\|}
\newcommand{\poly}{{\rm poly}}

\title{Calibration for the (Computationally-Identifiable) Masses}
\author{\'{U}rsula {H\'ebert-Johnson\thanks{\href{mailto:uhebertj@stanford.edu}{uhebertj@stanford.edu}}}\\Stanford University \and Michael P. Kim\thanks{\href{mailto:mpk@cs.stanford.edu}{mpk@cs.stanford.edu}. Supported in part by NSF grant CNS-122864.  Part of this work was completed while the author was visiting VMWare Research Group.} \\Stanford University\and Omer Reingold\thanks{\href{mailto:reingold@stanford.edu}{reingold@stanford.edu}. Supported in part by NSF grant CCF-1749750.}\\Stanford University \and Guy N. Rothblum\thanks{\href{mailto:rothblum@alum.mit.edu}{rothblum@alum.mit.edu} Research supported by the ISRAEL SCIENCE FOUNDATION (grant No. 5219/17).}  \\ Weizmann Institute}
\date{}

\begin{document}
\begin{titlepage}
\clearpage\maketitle
\thispagestyle{empty}

\begin{abstract}

As algorithms increasingly inform and influence decisions made about individuals, it becomes increasingly important to address concerns that these algorithms might be discriminatory. The output of an algorithm can be discriminatory for many reasons, most notably:
(1) the data used to train the algorithm
might be biased (in various ways) to favor certain populations over
others; (2) the \emph{analysis} of this training data might
inadvertently or maliciously introduce biases that are
not borne out in the data.  This work focuses on the
latter concern.

We develop and study \emph{multicalbration} -- a new measure of
algorithmic fairness that aims to mitigate concerns about discrimination that is introduced in the process of learning a predictor from data. Multicalibration guarantees accurate (calibrated) predictions for every subpopulation that can be identified within a specified class of computations. We think of the class as being quite rich; in particular, it can contain many overlapping subgroups of a protected group. 

We show that in many settings this strong notion of protection from discrimination is both attainable and aligned with the goal of obtaining accurate predictions. Along the way, we present new algorithms for learning a multicalibrated predictor, study the computational complexity of this task, and draw new connections to computational learning models such as agnostic learning.

\end{abstract}
\end{titlepage}

\section{Introduction}
Fueled by rapidly growing data sets and by breakthroughs
in machine learning, algorithms are informing
decisions that affect all aspects of life.  From
news article recommendations to criminal sentencing decisions
to healthcare diagnostics, increasingly algorithms
are used to make predictions about individuals.
Often, the predictions of an algorithm form the basis
for deciding how to treat these individuals
(suggesting a conservative Op-Ed, approving early parole, or
initiating chemotherapy).
A potential risk is that these algorithms -- even given access to
unbiased ground truth data --
might discriminate against groups of individuals that are protected by law or by ethics. This paper aims to mitigate such
risks of algorithmic discrimination.

We consider algorithms that predict the probabilities of events occurring for individuals. For example, a financial institution may
be interested in predicting the probability that an
individual will repay a mortgage.  The institution may
have at its disposal a large array of information
for each individual (as well as historic data and
global information such as financial and political
trends). But as thorough as the company may be, a significant
level of uncertainty will remain. Just as we wouldn't
expect to be able to predict with absolute certainty
whether it will rain on a particular day a year from now,
the financial institution wouldn't expect to predict
with absolute certainty whether an individual will
repay a loan. Thus, we consider algorithms that output, for every individual $i$, 
a prediction $x_i$ of the probability
that the event will occur for $i$; we call
the mapping from individuals to probabilities
a {\em predictor}.\footnote{A predictor can be used for regression
or, when paired with randomized rounding, for binary classification.}

Our focus in this paper is mitigating biases that may
arise as an algorithm \emph{analyzes} given data -- specifically, as the algorithm learns a predictor from
data.  Continuing the above example, suppose
that in a particular protected community $S$, on average,
individuals are financially disadvantaged and are
unlikely to repay a loan.  A machine-learning algorithm
that aims to optimize the institution's returns might
devote resources to learning outside of $S$ -- where there
is more opportunity for gains in utility -- and assign
a fixed, low probability to all $i \in S$.  Such an
algorithm would discriminate against the {\em qualified} members
of $S$. If $S$ is an underrepresented
subpopulation, this form of discrimination has the potential
to amplify $S$'s underrepresentation by refusing to
approve members that are capable of repaying the loan.

\subsection{Overview of our contributions}
Focusing on such concerns, we investigate a notion we call \emph{multicalibration}.
Multicalibration aims to mitigate discrimination that arises in the process of learning a predictor from given data. In a nutshell, multicalibration guarantees highly-accurate predictions for every subpopulation of individuals identified by a specified collection
$\C$ of subsets of individuals.
While our results hold for any set system $\C$, one natural
way to think about $\C$ is as a collection of subsets defined by
a family of boolean functions.
That is, for each $S \in \C$, we consider a function $c_S:\X \to
\set{0,1}$, where $c_S(i) = 1$ if and only if $i \in S$.
In particular, we will be most interested in considering families
of boolean functions that can be represented by some bounded
computational circuit class --
for instance, conjunctions of a small number of boolean features
or small decision trees.
In the mortgage repayment example above, if the class of qualified members of $S$ can be identified by a circuit $c \in \cal C$,
then the predictions made on qualified members of $S$ must be accurate, and the prediction algorithm cannot ignore / discriminate against these individuals. We emphasize that the class $\cal C$ can be quite rich and, in particular, can contain many overlapping subgroups of a protected group $S$.  In this sense, multicalibration captures the notion of
calibration for all \emph{computationally-identifiable} subsets,
where the notion of computational-identifiability is parameterized
by the expressiveness of $\C$.

We present a simple, general-purpose algorithm for learning a predictor
from a small set of labeled examples
that is multicalibrated with respect to {\em any} given class $\cal C$.
The algorithm is an iterative method, similar to gradient descent, and
can be viewed as a form of no-regret online optimization.
While the powerful online learning framework is exceptional in its
broad applicability, a number of subtleties arise when learning a
multicalibrated predictor that we must deal with.  In particular,
the definition of calibration refers to sets of individuals defined
by the values output by the predictor.  This leads to two key
challenges.  First, we must deal with an appropriate notion of
discretization of the range of real-valued outputs; this challenge
is mostly technical.  Secondly, the fact that we want to update
sets of predictions defined by the current predictions means that
any reduction to the online learning framework is
inherently an \emph{adaptive} analysis procedure; as such we need
to bring in machinery from the literature on differential privacy
and adpative data analysis in order to guarantee good generalization
from a small number of examples.

While we place no restrictions on the model class of the learned 
predictor, we show that \emph{implicitly} the algorithm
learns a model that provably generalizes well to unseen data, which may
be of independent interest.
We demonstrate this implicit generalization by viewing the learning
algorithm as an adaptive data analysis algorithm and showing
that the predictions we learn are highly compressible.
% \todo{connection here.  ``compressible''?  trying to translate the
% language of circuit complexity into ML language.}
In the language of circuit complexity, we show that we can build
a circuit, only slightly larger than the circuits from $\C$, that
implements the learned predictor.  As a corollary, the learned
predictor is efficient in both space to represent and time to evaluate.

% The complexity of evaluating the predictor is only slightly larger than the complexity of evaluating circuits in $\cal C$. The learning algorithm's running time depends linearly on the size of $\cal C$. 

We also study the {\em computational complexity} of learning multicalibrated predictors for more structured classes $\C$.
We show a strong connection between the complexity of learning a
multicalibrated predictor and agnostic learning \cite{haussler,agnostic}.
In the positive direction, if there is an efficient (weak) agnostic learner \cite{kalai2008agnostic,feldman2010distribution} for a class $\cal C$, then we can achieve similarly efficient multicalibration with respect to
sets defined by $\cal C$. In the other direction, we show
that learning a multicalibrated predictor on sets defined by $\C$ is as hard
as weak agnostic learning on $\C$.  In this sense, the complexity of
learning a multicalibrated predictor with respect to a class $\C$
is equivalent to the complexity of weak agnostic learning
on $\C$.

Finally, we demonstrate that the goal of multicalbration is aligned with the goal of achieving high-utility predictions. In particular, given any predictor $h$, we can use post-processing to obtain a multicalibrated predictor $x$ whose accuracy is no worse than that of $h$ (accuracy is measured in $\ell_2^2$ distance from the benchmark $p^*$). The complexity of evaluating the predictor $x$ is only slightly larger than that of $h$.
In this sense, unlike many fairness notions, multicalibration is not
at odds with predictive power and can be paired with any predictive
model at essentially no cost to its accuracy.

\paragraph{Organization.} We begin by elaborating on our setting. The remainder of the introduction is structured as follows. In Section~\ref{sec:intro:multi-calibrate} we elaborate on the notion of multicalibration and on its relationship to other notions in the larger context of fairness. We outline our main results on learning multicalibrated predictors in Section~\ref{sec:intro:contributions}. We further elaborate on related work and on future directions in Section~\ref{sec:intro:related}. Finally, we provide a brief overview of techniques in Section~\ref{sec:intro:techniques}.

\subsection{Fairness, calibration, and multicalibration}
\label{sec:intro:multi-calibrate}

\paragraph{High-level setting.} For an individual $i$ from the population $\X$,
we denote $i$'s outcome by $o_i \in \{0,1\}$. We take $p^*_i \in [0,1]$ to be the probability of outcome
$o_i=1$, conditioned on all the
information which is available to the algorithm. We denote
by $p^*$ the vector of these probabilities for all
individuals in $\X$. Our goal is to make a prediction $x_i$ for the value of $p^*_i$ for every individual $i$. As discussed above, we would like to avoid additional malicious or inadvertent discrimination (beyond the biases contained in the data). Thus, we refer to $p^*$ as the {\em benchmark predictor} for measuring discrimination. 

\paragraph{Calibration}

If we do not want a predictor $x$ to downplay the fitness of a group $S$, we can require that it be (approximately) accurate in expectation over $S$; namely, that $\card{\E_{i \sim S}\big[x_i - p_i^*\big]} \le \alpha,$ where $\alpha\ge 0$ is small. This means that the expectation of $x$ and $p^*$ over $S$ are almost identical. Calibration, introduced as a fairness concept in \cite{kleinberg2016inherent}, strengthens this requirement by essentially asking that for any particular value $v$, if we let $S_v = \set{i\in S: x_i = v}$ be the subset of $S$ of individuals with predicted probability $v$, then $\card{\E_{i \sim S_v}\big[x_i - p_i^*\big]}
= \card{v - \E_{i \sim S_v}[p_i^*]} \le \alpha.$\footnote{Calibration is often defined in the literature with respect to the instantiations of the events rather than their probabilities. Namely, that for every $v$, the fraction of $i$ in $S_v$ with $o_i=1$ is roughly $v$. As long as $S_v$ is sufficiently large and the instantiations are sufficiently independent, the two definitions are equivalent (up to small sampling errors) by concentration bounds. Necessarily, the formal definition given in Supplementary Materials {\bf A} will allow for a small fraction of the elements in $S$ to be misclassified due to being in a small set $S_v$.}

While calibration already precludes some forms of discrimination, as a group notion of fairness, it still allows for others (even if we assume that $p^*$ is perfectly fair). Indeed, weaknesses of group notions of fairness were discussed in \cite{DworkHPRZ12} (for a somewhat related notion called statistical parity), as a motivation for introducing an individual notion of fairness (see further discussion and comparisons below). A specific way to discriminate while satisfying calibration is to assign every member of $S$ the value $\E_{i \sim S}[p_i^*]$. While being perfectly calibrated over $S$, the qualified members of $S$ with large values $p^*_i$ will be hurt.

\paragraph{Balance} Other notions of fairness have been studied, looking at the rate of false positives and false negatives of predictions. Several variants of such properties have been recently studied on their own and in connection to calibration \cite{kleinberg2016inherent, chouldechova2017fair, pleiss2017fairness, hardt2016equality, corbett2017algorithmic}.
Let us briefly consider the notions referred to as balance in~\cite{kleinberg2016inherent}:
{\em balance for the positive class} --- the expected prediction $x_i$ for yes instances ($o_i=1$) in group $S$  equals the expected prediction for yes instances outside of $S$; and
{\em balance for the negative class} --- the expected prediction for no instances ($o_i=0$) in group $S$  equals the expected prediction for no instances outside of $S$.
In~\cite{hardt2016equality} it is shown how to obtain equalized odds, a definition related to error-rate balance, as a post-processing step of ``correcting'' any other predictor.

While both calibration and balance (as well as other related variants) intuitively seem like good properties to expect in a fair predictor (even if they are a bit weak), it has been shown that calibration and balance are impossible to obtain together (in non-degenerate cases) \cite{kleinberg2016inherent,chouldechova2017fair,pleiss2017fairness}.
This inherent conflict between balance and calibration, combined with
our observation that calibration is always aligned with the goal of
accurate high-utility predictions, implies that at times, balance
must be at odds with obtaining predictive utility.
Our approach in this paper towards mitigating the conflict between calibration and balance is to strengthen the protections implied by calibration, rather than enforcing balance.
While there are certainly
contexts in which ``equalizing the odds'' across groups is a good idea,
there are also contexts where calibration is a more appropriate notion
of fairness.
% For an elaborated discussion of balance and calibration
% as fairness concepts, see the discussion of corrective discrimination in Section~\ref{sec:intro:related}.
In particular, the benchmark predictor $p^*$ itself is unlikely to be balanced. Indeed, in our setting, the fact that balance is not satisfied might simply be an artifact of the inherent randomness in the process of sampling the outcome $o_i$, and this motivates our definition of multicalibration.

To illustrate this point, consider the following (intentionally-artificial) example: an algorithm is tasked with predicting the probability of rain during 10 days of winter, in two cities, a year from now. In city $A$ the algorithm predicts rain on each day with probability $0.8$; in city $B$ it predicts rain with probability $0.2$ (note that the certainty of the predictions is identical for both cities). A year passes and indeed in city $A$ it rains on 8 of the days, whereas in city $B$ it rains on only 2 days. What surprising  accuracy! Nevertheless, the predictions violate balance; that is, \emph{after observing which days rained},
the predictions treat the positive (respectively, negative) examples in city $A$ and $B$ differently.
Indeed, the mayor of city $A$ may complain that the predictor hurt tourism to her city: ``our sunny days are just as sunny as the sunny days of city $B$, so why were you so much more pessimistic about ours?'' The point is that \emph{a priori} the sunny and rainy days in $A$
(respectively, $B$) were indistinguishable from one another; further, there was a noticeable
difference in the likelihood of rain between $A$ and $B$.
Given the inherent uncertainty, it is unreasonable to expect accuracy on a subgroup that is {\em only identifiable a posteriori}.  While this example is overly-simplified, it points to the
fact that in any prediction context where there is inherent unpredictability, different false negative (or false positive) rates between groups are not necessarily a sign of discrimination.

\paragraph{Multicalibration}
As illustrated, a principle weakness of calibration as a
notion of fairness is that the guarantees are too coarse.
As far as we are aware, in the existing literature, calibration has
been applied to large, often disjoint, sets of protected groups;
that is, the guarantees are only required to hold on average over
a population defined by a small number of sensitive attributes, like
race or gender.  A stronger definition of fairness would ensure that
the predictions on \emph{every} subpopulation would be protected,
including, for instance, the qualified members of $S$ from the example
above.
The problem with such a notion is that it is information-theoretically
unattainable from a small sample of labeled examples, as it would
essentially require perfect predictions.
As such, we need an intermediary definition that balances the desire
to protect important subgroups and the information bottleneck that arises
when learning from a small sample.

To motivate our notion, consider an algorithm that produces a predictor $x$. The outcomes $o_i$ are determined, and then an auditor comes up with a set $S'$ that over-performed compared with the predictions of $x$. Perhaps the learning algorithm was lazy and neglected to identify the higher potential in $S'$? Perhaps the individuals of $S'$ were simply lucky? How can we tell?
To answer these questions, we take the following perspective:  on the one hand, we can only expect a learner to produce a predictor that is calibrated on sets that could have been identified \emph{efficiently} from the data at hand; on the other hand, we expect the learner to produce a predictor that is calibrated on \emph{every} efficiently-identifiable subset.
This motivates our definition of multicalibration, which loosely says:
``A predictor $x$ is $\alpha$-multicalibrated with respect to a family of sets $\cal C$ if it is $\alpha$-calibrated with respect to every $S \in \cal C$."

In the spirit of the discussion above, we take ${\cal C}$ to be a family of (sufficiently large) sets of individuals, such that for every $S\in {\cal C}$, the predicate $i\in S$ can be evaluated from the individual's data within a particular family of computations (conjunctions of four attributes, bounded-depth decision trees, or any other bounded complexity class). The more powerful the family, the stronger the guarantee becomes; no subpopulation that can be identified by the family will be overlooked.
At the extreme, consider multicalibration with respect to the family of polynomial-size circuits; in this case, every efficiently-identifiable subpopulation is protected! Note that the subpopulations in ${\cal C}$ can be overlapping, with complex relationships. In particular, they may well have no explicit dependence on sensitive attributes. In this sense, multicalibration goes far beyond calibration for several sensitive groups.

\subsection{Our Results}
\label{sec:intro:contributions}

Our study of multicalibration follows two major themes:
\begin{itemize}
\item We investigate the feasibility
of obtaining multicalibration; specifically,
we study the learnability of multicalibrated
predictors, showing both positive and negative
results.

\item We investigate the properties of multicalibrated
predictors.
While multicalibration provides strong guarantees against
forms of discrimination, we show that this protection can come
at little cost in terms of complexity and utility.
\end{itemize}

We begin with a high-level overview of our setup.
For a formal description of our model and assumptions,
see Section~\ref{sec:prelim}.
Suppose that for some universe of individuals $\X$,
we wish to predict whether some event (ad click, loan
repayment, cancer diagnosis, etc.) will occur for
each individual $i \in \X$.  We assume that for each
individual, there is some true underlying
probability $p_i^*$ that the event will occur.
We call any mapping from the universe to probabilities
a \emph{predictor};
formally, a predictor is a function\footnote{
We will interchange between function and vector
notation; generally we will denote the prediction that
$x$ assigns to an individual $i \in \X$ as $x_i$.}
$x:\X \to [0,1]$ that maps individuals from the universe
to estimates of their true probabilities.
We denote by $p^*$ the benchmark predictor that gives the
true probabilities.

The benchmark predictor is, itself, multicalibrated with
respect to every collection of subsets $\C$.  Thus,
if we can efficiently learn a predictor from the data
at hand with sufficient accuracy across the entire population (specifically, with small $\ell_1$ distance from $p^*$),
then the learned predictor will be multicalibrated.
That said, in most interesting situations, $p^*$ will be too complex to  learn efficiency with such uniform accuracy (especially given that the values of $p^*$ themselves will usually not be observable).
Focusing on such settings, we aim for multicalibration as a notion of protection from discrimination.

The first question to address
is whether multicalibration is feasible.
For instance, it could be the case that the requirements
of multicalibration are so strong that they would require
learning and representing an arbitrarily complex
function $p^*$ exactly, which we've established can be
infeasible.  Our first result characterizes
the complexity of representing a multicalbrated predictor.
We demonstrate that multicalibration, indeed, can be
achieved efficiently: for any $p^*$ and any collection
of large subsets $\C$, there exists a predictor that is
$\alpha$-multicalibrated on $\C$, whose complexity is
only slightly larger than the complexity required to describe
the sets of $\C$. For concreteness, we use circuit size as our measure of complexity in the following theorem. 

\begin{result}
\label{result:ckt}
Suppose $\C \subseteq 2^{\X}$ is collection of sets
where for $S \in \C$, there is a circuit of size $s$
that computes membership in $S$ and $\card{S} \ge \gamma \card{\X}$.
For any $p^*:\X \to [0,1]$,
there is a predictor that is
$\alpha$-multicalibrated with respect to $\C$ implemented by
a circuit of size $O(s/\alpha^4\gamma)$.
\end{result}

As stated, this result claims the existence of multicalibrated
predictors whose predictions are efficient to evaluate.
The existence of such predictors, while interesting
from a complexity-theoretic perspective, begs the more
practical question of whether we can get our hands on
such a predictor.

In fact, we prove Theorem~\ref{result:ckt} algorithmically
by learning an
$\alpha$-multicalibrated predictor from labeled
samples.  While our model assumes the existence of
some underlying true probabilities $p^*$, in most
applications, these probabilities will not be
directly observable.  As such, we design
algorithms that learn predictors from samples
of individuals labeled with their \emph{outcomes};
specifically, we assume access to labeled samples $(i,o_i)$
of individual-outcome pairs, where $i$ is sampled
according to some distribution $\D$ on the universe and
$o_i$ is the realized value
of a Bernoulli trial with probability $p_i^*$.
Naturally, in this model, our goal is to give algorithms
that are efficient in terms of running time and
sample complexity.

In Section~\ref{sec:learning},
we give an algorithm for learning
a multicalibrated predictor from labeled samples,
whose running time scales linearly with $\card{\C}$
and polynomially with $\alpha$ and $\gamma$.
A consequence of our analysis is that naively,
the sample complexity can be upper-bounded by
$\log(\card{\C})/\alpha^6\gamma^6$.
We show how to improve the sample complexity
over the naive approach by polynomial factors
in both $\alpha$ and $\gamma$.
\begin{result}
\label{result:alg}
Suppose $\C \subseteq 2^{\X}$ is collection of sets
such that for all $S \in \C$, $\card{S} \ge \gamma \card{X}$,
and suppose set membership can be evaluated in time $t$.
Then there is an algorithm that
learns a predictor of $p^*:\X \to[0,1]$
that is $\alpha$-multicalibrated on $\C$ from
$O(\log({\card{\C}})/\alpha^{11/2}\gamma^{3/2})$ samples
in time $O(\card{\C}\cdot t \cdot \poly(1/\alpha,1/\gamma))$.
\end{result}

Observing the linear dependence in the running time on $\card{\C}$, it is natural to try and develop a learning
procedure with subpolynomial, or even polylogarithmic, dependence on $\card{\C}$.  Our next results
aim to characterize when this optimistic goal is possible --
and when it is not.
We emphasize that the algorithm of Theorem~\ref{result:alg}
learns a multicalibrated predictor for \emph{arbitrary}
$p^*:\X \to [0,1]$ and $\C$.
In the setting where we cannot exploit structure in
$p^*$ to learn efficiently, we might hope
to exploit structure, if it exists, in the collection of
subsets $\C$.  Indeed, we demonstrate a connection between
our goal of learning a multicalibrated predictor and
weak agnostic learning, introduced in the literature on
agnostic boosting
\cite{ben2001agnostic,kalai2008agnostic,kalai2009,feldman2010distribution}.
Our next result shows that efficient weak agnostic
learning over $\C$ implies efficient learning of
$\alpha$-multicalibrated predictors on $\C$.
\begin{result}[Informal]
\label{result:boosting}
If there is a weak agnostic learner for $\C$
that runs in time $T$, then there is an algorithm for
learning an $\alpha$-multicalibrated predictor on
$\C' = \set{S \in \C: \card{S} \gamma \card{X}}$
that runs in time $O(T\cdot\poly(1/\alpha,1/\gamma))$.
\end{result}
Slightly more formally,
we require a $(\rho,\tau)$-weak agnostic learner
in the sense first introduced by \cite{kalai2008agnostic}
and generalized by \cite{feldman2010distribution}.  For the
specifics of the requirements and parameters,
see the formal statement in Section~\ref{sec:weak}.

These results show that under the right structural assumptions
on $p^*$ or on $\C$, a multicalibrated predictor may be
learned more efficiently than our upper bound for the general case.
Returning to the general case, we may wonder if these
structural assumptions are necessary; we answer this question
in the positive.  We show that for worst-case $p^*$
learning a multicalibrated predictor on $\C$ is as hard
as weak agnostic learning for the class $\C$.
\begin{result}[Informal]
\label{result:learner}
If there is an algorithm for learning an $\alpha$-multicalibrated
predictor on a collection of sets $\C' = \set{S \in \C: \card{S} \ge
\gamma N}$ that runs in time $T$, then
there is an algorithm that implements a $(\rho,\tau)$-weak agnostic learner
in time $O(T\cdot \poly(1/\tau))$ for
any $\rho > 0$ where $\tau = \poly(\rho,\gamma,\alpha)$.
\end{result}
In general, agnostic learning is considered a notoriously
hard computational problem.  In particular, under cryptographic
assumptions \cite{valiant1984theory,ggm,prfs}, this result
implies that there is some constant $t > 0$, such that
any algorithm that learns an $\alpha$-multicalibrated
predictor requires $\Omega(\card{\C}^t)$ time.

Finally, we return our attention to investigating the
utility of multicalibrated predictors.  Above, we have
argued that multicalibration provides a strong protection
of groups against discrimination.  We show that this
protection comes at (next to) no cost in the utility of the
predictor.
This result adds to the growing literature on fairness-accuracy
trade-offs \cite{fish2016confidence,berk2017convex,chouldechova2017fairer}.
\begin{result}
\label{result:best}
Suppose $\C \subseteq 2^{\X}$ is a collection of subsets of
$\X$ and $\H$ is a set of predictors.
There is a predictor $x$ that is $\alpha$-multicalibrated
on $\C$ such that
$$\E_{i \sim \X}[(x_i-p_i^*)^2] - \E_{i \sim \X}[(h_i^* - p_i^*)^2] < 6\alpha,$$
where $h^* = \argmin_{h \in \H}\E_{i \sim \X}[(h - p^*)^2]$.
Further, suppose that for all $S \in \C$, $\card{S} \ge \gamma N$, and suppose
that set membership for $S \in \C$ and $h \in \H$
are computable by circuits of size at most $s$; then $x$ is computable
by a circuit of size at most $O(s/\alpha^4\gamma)$.
\end{result}
We can interpret Theorem~\ref{result:best} in
different ways based on the choice of $\H$.
Suppose there is some sophisticated learning algorithm
(say, a neural network) that produces some predictor $h$ that
obtains exceptional performance, but may violate calibration arbitrarily.
If we take $\H = \set{h}$, then this result says:
enforcing calibration on $h$ after learning does not hurt
the accuracy by much.
Further, our proof will demonstrate that if calibration
changes the predictions of $h$ significantly, then this
change amounts to an \emph{improvement} in accuracy.
See Lemma~\ref{lem:best} for the exact statement.
Taking a different perspective, we can also think of $\H$ as a
set of predictors that, say, are implemented by a
circuit class of bounded complexity (e.g.~conjunctions of $k$
variables, halfspaces, circuits of size $s$).
This theorem shows that
for any such class of predictors $\H$ of bounded complexity,
there exists a multicalibrated predictor with similar 
complexity that performs as well as the best $h^* \in \H$.
In this sense, with just a slight overhead in complexity,
multicalibrated predictors can achieve ``best-in-class''
predictions.

\subsection{Related work}
\label{sec:intro:related}

\paragraph{Calibration}  As mentioned earlier, calibration is a
well-studied concept in the literature on statistics and econometrics,
particularly forecasting.
For a background on calibration in this context, see \cite{sandroni2003calibration,foster2015smooth} and the references
therein.
Calibration has also been studied in the context of structured
predictions where the supported set of predictions is large
\cite{kuleshov2015calibrated}.
Our algorithmic result for multicalibration bears similarity to
works from the online learning literature
\cite{blum2007external,khot2008minimizing,trevisan2009regularity}.
In Section~\ref{sec:AE}, we describe a procedure
for achieving a weaker notion than multicalibration, which we refer
to as multi-``accurate in expectation" (multi-AE); our result can be viewed as
a restatement of the boosting-based proof due to \cite{trevisan2009regularity}
that every high-entropy distribution is indistinguishable by
small circuits from an efficiently-samplable distribution of the
same entropy.  While the algorithms for learning a multi-AE and
multicalibrated predictor bear similarity, when we
turn to achieving multicalibration, we have to deal with a number
of nontrivial challenges due to learning from a small sample (while
guaranteeing generalization) and using an apprropriate notion of
discretization that are not necessary in \cite{trevisan2009regularity}.
We are also not aware of any prior works connecting the literature on
calibration with the literature on differential privacy and adaptive
data analysis.

\paragraph{Between populations and individuals}
In \cite{DworkHPRZ12}, an individual notion of fairness was defined, referred to as ``fairness through awareness.'' This notion relied on a task-specific metric of similarity between individuals and formalized the idea that similar individuals (under this metric) are treated similarly. It is natural, in an array of applications, to view $p^*$ as defining a metric -- two individuals $i$ and $j$ will be assigned the distance $|p^*_i-p^*_j|.$ In this work we consider cases in which figuring out $p^*$ in its entirety is difficult. Thus, one can view our approach as a meaningful compromise between group fairness (satisfying calibration) and individual-calibration (closely matching $p^*_i$).
The multicalibration framework presented in this
work served as the inspiration for subsequent work of a subset of the authors
\cite{ftba} investigating how to interpolate between
statistical and individual notions of ``metric fairness'' for
general similarity metrics.

\paragraph{Subgroup Fairness}
Contemporary independent work of \cite{kearns2017preventing} also investigates strengthening the guarantees of notions of group fairness by requiring that these properties hold for a much richer collection of sets. Unlike our work, their definitions require balance or statistical parity on these collection of sets. Their motivation is similar to ours, namely to bridge the gap between notions of individual fairness (powerful but hard to obtain) and population-level fairness (easy to work with but weak).

Despite similar motivations, the two approaches to subgroup fairness
differ in substantial ways. As a concrete example, Theorem~\ref{result:best}
demonstrates that achieving multicalibration is aligned with the incentives
of achieving high-utility predictors; this is not necessarily the case
with balance-based notions of fairness.
Indeed, in the setting considered in this work, one of the motivations for multicalibration is a critique of balance that may only be heightened when considering ``multi-balance". Consider the example in \cite{DworkHPRZ12} where in a population $S$ the strongest students apply to Engineering whereas in the general population $T$ they apply to Business. Even if predictor $p^*$ (for the probability of success in school) is balanced  between $S$ and $T$, forcing balance between the two populations within Business applicants and within Engineering applicants would be unfair to qualified applicants in both groups.
(For more discussion, see the section below on ``Corrective discrimination".)

On a technical level, both works draw connections between agnostic learning and the task of finding a group on which the fairness condition is violated (~\cite{kearns2017preventing} refer to this as auditing). Leveraging this connection, we show how to use an agnostic learner to learn a multicalibrated predictor efficiently. In a similar vein, \cite{kearns2017preventing}
show how to use an agnostic learner to efficiently learn a classifier that satisfies balance or parity on a collection of sets, while achieving competitive accuracy within a given class. They also implement a more-practical variant of their algorithm and use it to conduct empirical experiments.

Preventing discrimination by algorithms is subtle; different scenarios
will call for different notions of protection.
As such, it is hard to imagine a universal definition of fairness.
Nevertheless, these two independent works validate the need to investigate
attainable approaches to mitigating discrimination beyond large protected
groups.  It will be interesting to further understand how these notions of
subgroup fairness relate to one another, and when each approach is
most appropriate.

\paragraph{Calibration for multiple protected sets vs.\ multicalibration} The literature on fairness commonly considers more than a single protected set (cf.\ \cite{celis2017ranking} which studies fairness in ranking algorithms, with protected groups being defined by each attribute of interest). The major difference in our work is that we think of protected groups as any group that can be efficiently identified, rather than those defined by sensitive properties. One side benefit of the generality of our approach is that it does not single out groups based on sensitive values for special treatment (which can be illegal in some contexts).

\paragraph{Calibration, Bandits, and Regret}
There is a growing literature on designing fair selection policies
in the multi-arm bandits setting \cite{joseph2016fairness,joseph2017fair,liu2017calibrated}.
Recently \cite{liu2017calibrated} initiated the study of
calibration in this context. Motivated by the aforementioned
work of \cite{DworkHPRZ12}, their notion of calibrated fairness
aims to ``treat similar individuals similarly'' by designing
sampling policies that have low \emph{fairness regret}.

\paragraph{Causality and discrimination in the data}

Discrimination can occur at any stage throughout the algorithmic pipeline
and in many forms.
The most important aspect of developing a theory of algorithmic fairness that multicalibration does not address is ``unfair'' data (see more below). Kilbertus {\em et al.} \cite{kilbertus2017avoiding} voice an important criticism of any notion of fairness based solely on observational criteria, as discrimination can occur through unresolved causal relationships. This criticism applies to the notions of balance, calibration, and multicalibration, which all depend only on the joint distribution of predictions, outcomes, and features.

Rather than attempting to provide a general-purpose definition of fairness, our work tackles a particular concern about discrimination that can occur as part of the process of learning a predictor from given data. In this context, we believe that multicalibration provides an important anti-discrimination guarantee.

\paragraph{Corrective discrimination}

Let us look deeper into the biases that may be present in the gathered data, by considering the mortgage example again: perhaps the number of members of $S$ that received loans in the past is small (and thus there are too few examples for fine-grained learning within $S$); perhaps the attributes are too limited to identify the qualified members of $S$ (taking this point to the extreme, perhaps the only available attribute is membership in $S$).
In these cases, the data may be insufficient for multicalibration
to provide meaningful guarantees.  Further,
even if the algorithm was given access to unlimited rich data
such that refined values of $p^*$ could be recovered,
there are situations where preferential treatment may be in order: after all, the salaries of members of $S$ may be lower due to historical discrimination.

For these reasons, our concern that balance is inconsistent with $p^*$ could be answered with: ``yes, and purposely so!'' Indeed, \cite{hardt2016equality} promotes enforcing a balance-related property, called equalized odds, as a form of ``corrective discrimination.'' While this type of advocacy is important in many settings, multicalibration represents a different addition to the quiver of anti-discrimination measures, which we also believe is natural and desirable in many settings.

Consider a concrete example where multicalibration is appropriate, but equalizing error rates might not be: Suppose a genomics company offers individuals a prediction of their likelihood of developing
certain genetic disorders.  These disorders have different
rates across different populations; for instance, Tay-Sachs
disease is rare in the general population, but occurs much
more frequently in the Ashkenazi Jewish population.
We certainly do not want to enforce corrective discrimination on the Ashkenazi population by down-weighting the prediction that individuals would have Tay-Sachs (as they are endogenously more likely to have the disease). However, we also don't want the company to base its prediction solely on the Ashkenazi feature (either positively or negatively). Instead, enforcing multicalibration would require that the learning algorithm investigate both the Ashkenazi and non-Ashkenazi population to predict accurately in each group (even if this means a higher false positive rate in the Ashkenazi population). In this case, relying on $p^*$ seems to be well-aligned with promoting fairness.

Finally, we consider the interplay between multicalibration and ``corrective discrimination'' to be an important direction for further research. For example, one can imagine applying corrective measures, such as the transformation of \cite{hardt2016equality},  to a multi-calibrated predictor.

\subsection{Our Techniques}
\label{sec:intro:techniques}
Here, we give a high-level technical overview of our results.
Our techniques draw from the literature on computational learning theory,
online optimization, differential privacy, and adaptive data analysis.

\paragraph{Learning a multicalibrated predictor}
In Section~\ref{sec:alg}, we describe an algorithm for
learning $\alpha$-multicalibrated predictors as stated
in Theorem~\ref{result:alg}.  Our algorithm
is an iterative procedure.  In particular, we will
maintain a candidate predictor $x$, and at each
iteration, the algorithm corrects the candidate values of
some subset that violates calibration until the candidate
predictor is $\alpha$-multicalibrated.
Recall that calibration over a set $S$ requires that on
the subsets
$S_v = \set{i \in S : x_i = v}$ (which we will refer
to throughout as \emph{categories}), the expected value
of the true probabilities $\E_{i \sim S_v}[p_i^*]$
on this set is close to $v$.  As such,
the algorithm is easiest to describe in the
statistical query model, where we query for
noisy estimates of the true statistics on subsets of
the population and update the predictor based on these
estimates.  In particular, given a statistical query
oracle that guarantees tolerance $\omega = O(\alpha\gamma)$,
the estimates will be accurate enough to guarantee
$\alpha$-calibration on sets $S$ with $\card{S} \ge \gamma \card{X}$.

When we turn to adapting the algorithm
to learn from random samples, the algorithm answers
these statistical queries using the empirical estimates
on some random sample from the population.
Standard generalization arguments \cite{kearnsvazirani}
show that if the set of queries we might
ask is fixed in advance, then we could bound the
sample complexity needed to answer these
\emph{non-adaptive} queries as
$\tilde{O}(\log{\card{\C}}/\omega^2)$.
Note, however, that the categories $S_v$ whose
expectations we query are selected
\emph{adaptively} (i.e.~with dependence on the results of
prior queries). In particular, the definition of the
categories $S_v$ depends on the current values of the
predictor $x$; thus, when we update $x$ based on the result
of a statistical query, the set of categories on which we
might ask a statistical query changes.
In this case, we cannot simply
apply concentration inequalities and take a union bound
to guarantee good generalization without resampling
every time we update the predictor.

To avoid this blow-up in
sample complexity, we appeal to recently-uncovered connections
between differential privacy and adaptive data analysis developed in
\cite{dwork2015preserving,dwork2015preserving,bnsssu15,dwork2015reusable}.
In Section~\ref{sec:privacy}, we show how to answer
the statistical queries in a
way that guarantees the learning algorithm is differentially
private.  This, in turn, allows us to argue that the
predictor our algorithm learns from a small sample
will be multicalibrated, not just for the observed
individuals but also the unseen individuals.
In particular, using a privacy-based approach
with an analysis tailored to the goal of calibration,
we obtain sample complexity
that depends on $1/\alpha^{11/2}\gamma^{3/2}$
as opposed to the naive approach which results
in $1/\alpha^6\gamma^6$.

As stated earlier, Theorem~\ref{result:ckt} can be seen
as a corollary of Theorem~\ref{result:alg}. Bounding the
number of iterations needed by the algorithm to converge
not only upper-bounds the algorithm's running time
and sample complexity,
but also implies that the circuit complexity of the
learned predictor is not much larger than the complexity
of evaluating membership for $S \in \C$.  We explain this
implication in Section \ref{sec:ckt}.

\paragraph{The complexity of multicalibration}
In Section~\ref{sec:runtime}, we discuss
the complexity of learning a multicalibrated predictor and
draw connections to agnostic learning \cite{haussler,agnostic}.
We show that the algorithm learns a multicalibrated predictor
in a bounded number of iterations;
however, without additional assumptions about
$p^*$ or $\C$, each
iteration could take $\Omega(\card{\C})$ time.  In the
cases where $\card{\C}$ is large, we might hope to improve
the dependence on $\card{\C}$ to polylogarithmic or perhaps
subpolynomial.  If $p^*$ can be learned
directly, then we can eliminate the dependence on $\card{\C}$
and instead, only depend on the minimum $\gamma$ such that
for all $S \in \C$, $\card{S}\le \gamma N$.

In the case where $p^*$ is arbitrary, we show that
improving the dependence on $\card{\C}$ is possible if $\C$ is
structured in a certain way, drawing a connection to the
literature on agnostic learning
\cite{haussler,agnostic,kalai2008agnostic,feldman2010distribution}.
Recall, in our algorithm for learning multicalibrated
predictors, we maintain a candidate predictor $x$, and
iteratively search for some set $S \in \C$ on which $x$
is not calibrated.  To solve this search problem more quickly,
we frame the search as weak agnostic learning over a concept class
derived from $\C$ and over the hypothesis class of
$\H = \set{h:\X \to [-1,1]}$.

Specifically, consider the concept class defined by the
collection of subsets $\C$, where for each $S \in \C$,
we include the concept $c_S:\X \to \set{-1,1}$
where $c_S(i) = 1$ if and only if $i \in S$.
We show how to design a ``labeling'' $\ell:\X \to [-1,1]$
for individuals such that if
$x$ violates the calibration constraint on any $S \in \C$,
then the concept $c_S$
correlates nontrivially with the labels over the
distribution of individuals,
i.e. $\langle c_S, \ell \rangle \ge \rho$ for some
$\rho > 0$.

Thus, if $x$ is not yet multicalibrated on $\C$, then
we are promised that there is some concept $c_S$ with
nontrivial correlation with the labels; we observe
that this promise is exactly the requirement for a weak
agnostic learner, as defined in \cite{kalai2008agnostic,feldman2010distribution}.
In particular, given labeled samples $(i,\ell(i))$
sampled according to $\D$, if there is a concept $c_S$
with correlation at least $\rho$ with $\ell$, then
the weak agnostic learner returns a hypothesis $h$
that is $\tau$ correlated with $\ell$ for some
$\tau < \rho$.
The catch is that this hypothesis
may not be in our concept class $\C$, so we cannot
directly ``correct'' any $S \in \C$.  Nevertheless,
the labeling on individuals $\ell$ is designed such
that given the hypothesis $h$, we can still extract
an update to $x$ that will make global progress towards
the goal of attaining calibration.  As long as $\tau$
is nontrvially lower bounded, we can
upper bound the number of calls we need to
make to the weak learner.
The details of our choice of labels and how we track
progress are given in Section~\ref{sec:weak}.

Additionally, we show that our reduction to weak
agnostic learning is unavoidable.  In particular,
we show that if it is possible to learn a multicalibrated
predictor with respect to $\C$, then it is possible
to weak agnostic learn on $\C$ (if we view $\C$ as a concept
class). Specifically, we will show how to implement
a weak agnostic learner for $\C$, given an algorithm
to learn an $\alpha$-multicalibrated predictor $x$
with respect to $\C$ (in fact, we only need the
predictor to be multicalibrated on
$\C' = \set{S \in \C: \card{S} \ge \gamma \card{X}}$).
The key lemma for this reduction says that if there is
some $c \in \C$ that is nontrivially correlated with the labels,
then $x$ is also nontrivially correlated with $c$.
As there are many natural classes $\C$ for which agnostic
learning is conjectured to be hard, this gives a strong
negative result to the question of whether we can
obtain speedups in the general case.

In combination, these results show that the complexity of
learning a multicalibrated predictor with respect to
a class $\C$ is equivalent to the complexity of weak
agnostic learning $\C$.  Section~\ref{sec:weak}
contains the formal statements and proofs that imply this equivalence.

\paragraph{``Best-in-class'' prediction}
In Section~\ref{sec:best}, we turn to understanding how requiring a
predictor to be multicalibrated affects the accuracy
of the predictor.  
We show -- in contrast to many other
notions of fairness -- multicalibration does not limit
the utility of a predictor.
In particular, given
any collection of predictors $\H$, and any collection of
subsets $\C$, we design a procedure for obtaining
a predictor $x$ that is $\alpha$-multicalibrated on
$\C$ and achieves
expected squared prediction error less than or equal
to the the best predictor in $\H$ (plus a small additive
error on $\alpha$).  Further, leveraging Theorem~\ref{result:ckt}
and Theorem~\ref{result:alg}, we show that the predictor $x$
can be learned from samples and implemented by a circuit
of comparable size to the predictors $h \in \H$.
To prove that multicalibration does not negatively impact
the utility, we in fact, show a much stronger statement:
if applying multicalibration to some $h \in \H$ changes
the predictions of $h$ significantly (i.e. if
$\norm{x-h}$ is large), then this change represents
an improvement in accuracy (i.e. $\norm{x-p^*} < \norm{h-p^*}$).
In this sense, requiring multicalibration is aligned
with the goals of learning a high-utility predictor.

\paragraph{Organization of the paper}
In Section~\ref{sec:prelim}, we provide a description
of our model and the formal definitions related to multicalibration.
In Section~\ref{sec:learning}, we describe and analyze our
algorithm for learning multicalibrated predictors.
In Section~\ref{sec:weak}, we investigate the complexity
of obtaining multicalibration, showing a tight connection
to the complexity of weak agnostic learning.
Finally, in Section~\ref{sec:best}, we demonstrate how
multicalibration achieves ``best-in-class'' prediction.

\section{Preliminaries}
\label{sec:prelim}
\paragraph{Predictors}
Let $\X$ denote the universe of $N$ individuals over
which we wish to make predictions about some outcome
$o \in \set{0,1}^N$.
We assume that outcomes are the result of some underlying
random process, where for each $i$,
$o_i$ is sampled independently\footnote{In many cases, complete
independence may not hold; individuals' outcomes may be
correlated in nontrivial ways.
The only place we will use independence is to argue that
the reliable statistics can be estimated from the observable
data; our arguments can be applied to any model for which one
can prove the appropriate tail inequalities.
For further consideration, continue to our discussion
on ``observable'' calibration after
Claim~\ref{claim:calibrated}.}
as a Bernoulli random variable
with success probability $p_i^*$.
We aim to predict the underlying parameters of the
process, rather than the realized outcome of the process.
A \emph{predictor} $x: \X \to [0,1]$ of outcome $o$
is some mapping from individuals $i \in \X$ to $[0,1]$,
where $x_i$ is the prediction of $p_i^*$.  We refer to
$p^*$ as the baseline predictor.
When it is notationally convenient, we will sometimes
treat predictors as vectors $x \in [0,1]^N$
rather than as functions $x:\X \to [0,1]$,
where we assume a bijection from $\X$ to $[N]$.

\begin{remark}
Often in learning theory, we think of learning functions
$h:\F \to [0,1]$ over the space of possible features $\F$.
We find it preferable to reason about predictions about
\emph{individuals}; nevertheless,
in our model, we can think of $x_i$ as given by the
composition of two separate functions, where
$x_i = x'(\phi(i))$ for $x':\F \to [0,1]$ and $\phi:\X \to \F$.
As $\phi$ will be fixed and assumed to be some simple
function (say, mapping individuals to their
$\langle$ age, height, ZIP code, $\hdots \rangle$)
we drop the explicit reference to $\F$ and $\phi$.
\end{remark}

\paragraph{Sampling from $\X$}
Throughout, we will assume that our learning algorithms
have the ability to efficiently obtain
randomly sampled individuals from $\X$ (and by rejection
sampling, large subsets of $\X$).  Specifically,
we will use $i \sim S$ to denote sampling $i$ uniformly
at random from $S \subseteq \X$.
\begin{remark}
\label{rem:dist}
Our focus on the uniform distribution is mostly
syntactic; as per our distinction
between the identities of individuals $i \in \X$ and their
features $\phi(i)$, the uniform distribution over
individuals gives rise to a rich class of distributions
over the features of individuals.
\end{remark}
We will further assume that sampling $i \sim \X$ is
inexpensive compared
to obtaining a corresponding outcome $o_i \in \set{0,1}$
for $i \sim \X$ that is
Bernoulli distributed with parameter $p_i^*$.
As such, when measuring the sample complexity, we will
only count the latter type of labeled samples.
The learner has a good sense
of the distribution over features, but not of the outcomes
that arise from these features.

When measuring the accuracy of a predictor $x$,
we will use the squared prediction error $\norm{p^* - x}^2$,
as this measure of divergence penalizes large individual
deviations. The $\ell_2^2$ distance will also be useful
for measuring the similarity of predictors.
While we elect to use $\ell_2^2$, much of our
analysis could be performed using any Bregman divergence;
in particular, if we elected to work over arbitrary multinomial
distributions over individuals in $\X$, we could measure accuracy
in terms of the squared Mahalanobis distance from the optimal
predictor, i.e. $(p^* - x)^TD(p^* - x)$, where $D$ is the
diagonal matrix in which $D_{ii}$ is the probability
of sampling individual $i \in \X$.

\subsection{Calibration}
Next, we give formal definitions of the criteria
we use to measure algorithmic discrimination.
The first notion captures the idea
that, on average, we would like an algorithm's predictions to be unbiased.
\begin{definition}[Accurate in expectation]
For any $\alpha > 0$ and $S \subseteq \X$,
a predictor $x$ is $\alpha$-\emph{accurate in expectation
($\alpha$-AE)} with respect to $S$ if
$$\card{\E_{i \sim S}[x_i - p_i^*]} \le \alpha.$$
\end{definition}
That is, the predictions, averaged over the set $S$, are
accurate up to some additive $\alpha$ slack.  This
basic condition is necessary to ensure unbiased predictions,
but it is too weak to guarantee discrimination hasn't occurred.
In particular,
suppose the underlying probabilities are such that $p_i^* = 1/2$
for all $i \in S$.
A predictor that predicts $1$ on half of the individuals in $S$
and $0$ on the other half is
accurate in expectation, but it is arguably discriminatory;
there is no difference between the individuals in $S$,
but the predictor has artificially created two categories within
this population.  This example motivates the definition of
calibration.
Calibration mitigates this form of discrimination
by considering the expected values
over categories $S_v = \set{i: x_i = v}$ defined by the
predictor $x$.  Specifically,
$\alpha$-calibration with respect to $S$ requires that for all
but an $\alpha$-fraction of a set $S$, the average of the true
probabilities of the individuals receiving prediction $v$ is
$\alpha$-close to $v$.
\begin{definition}[Calibration]
\label{def:calibration}
For any $v \in [0,1]$, $S \subseteq \X$, and predictor
$x$, let $S_v = \set{i:x_i = v}$.  For $\alpha \in [0,1]$,
$x$ is $\alpha$-\emph{calibrated} with
respect to $S$ if there exists some $S' \subseteq S$
with $\card{S'} \ge (1-\alpha)\card{S}$ such that
for all $v \in [0,1]$,
$$\card{\E_{i \sim S_v\cap S'}[x_i - p_i^*]} \le \alpha.$$
\end{definition}
Note that $\alpha$-calibration with respect to $S$
implies $2\alpha$-AE with respect to $S$. To see this,
observe that calibration
implies that on a $(1-\alpha)$-fraction of $S$, the average
of the values are $\alpha$-close to the expectation on
this fraction; even if the other $\alpha$-fraction is
arbitrary, it can only introduce another additive $\alpha$ error.
This definition only requires the notion of calibration
to hold on a $(1-\alpha)$-fraction of each $S$; this is
for technical reasons.  See Definition~\ref{def:al-mult}
to understand better the motivations for this choice.

\paragraph{Departures from prior definitions}
This notion of calibration departs from prior definitions
in a few ways. Earlier definitions required
exact calibration with respect to $\X$;
we find it meaningful to consider approximate calibration.
Introducing approximation into the definition of calibration
has practical motivation. In particular, we don't expect to
know the exact probabilities, nor can we observe the entire
population.  Given a desired accuracy and access to samples
from the population, we can quantify how many samples are
needed to guarantee calibration with good probability.

Note that in our definition of $\alpha$-calibration, we
require \emph{relative} additive error that scales with
the size of $S_v$.  That is, for some $S$
and category $S_v = \set{i: x_i = v} \cap S$,
the magnitude of the sum of errors on $S_v$,
$\card{\sum_{i \in S_v}(v - p_i^*)}$, scales with $\card{S_v}$,
not $N$. This relative
approximation model prevents certain ``attacks'' against
approximate calibration.  Specifically, one natural way
a predictor might attempt
to sidestep the anti-discrimination properties of
approximate calibration would be to support many distinct
values of $v \in [0,1]$, each with a very small number
of individuals.  If our notion of approximation allowed
for \emph{absolute} errors (i.e. errors whose sum over
$S_v$ scale with $N$ instead of $\card{S_v}$), then this
attack would be viable, leading to essentially no guarantees

Another subtle distinction is that
we evaluate calibration
with respect to the underlying probabilities $p^* \in [0,1]^N$,
as opposed to the realized outcomes $o \in \set{0,1}^N$.
We refer to the earlier notion as
\emph{observable} calibration. Formally,
a predictor is \emph{observably $\alpha$-calibrated}
with respect to $S \subseteq \X$ and outcome $o \in
\set{0,1}^N$ if for all $v \in [0,1]$,
$$\card{\E_{i \sim S_v\cap S'}[x_i - o_i]} 
\le \alpha,$$
where $S_v$ and $S'$ are as in
Definition~\ref{def:calibration}.
In our terminology, earlier references
to calibration require that a predictor be
\emph{observably} $0$-calibrated with respect to each
protected group $S$.
Our introduction of a non-zero $\alpha$ in these definitions
allows us to relate our notion of calibration
and the previous notion of observable calibration.
\begin{claim}
Let $S \subseteq \X$, $\xi > 0$, and
$\alpha > \sqrt{\frac{\log(1/\xi)}{2\card{S}}}$.
Suppose the outcome $o \in \set{0,1}^N$ is drawn according
to the true probabilities $p^*$.
Then, with probability at least $1-\xi$ over the
draw of $o$,
any predictor that is $\alpha$-calibrated with respect
to $S$ will be \emph{observably} $2\alpha$-calibrated
with respect to $S$.
\label{claim:calibrated}
\end{claim}
Claim~\ref{claim:calibrated}
follows directly from an application of
Hoeffding's inequality over the random outcome $o$.
The claim implies that for sufficiently large $\alpha$,
to guarantee \emph{observable} $2\alpha$-calibration
with high probability, it suffices to consider our
notion of $\alpha$-calibration.

\paragraph{Independence and observability}
We note that to prove Claim~\ref{claim:calibrated},
we use the independence of $o_i$'s in our application
of Hoeffding's inequality.  Throughout, the only times
we invoke the
independence of the realization of the $o_i$'s are to
prove that the empirical statistics over a sufficiently
large sample of observations will be concentrated around
the corresponding statistics of the true parameters.
Thus, the independence assumption isn't
strictly necessary; our results should hold for any model
where the outcomes admit similar tail inequalities.
We note, however, that if strong dependencies exist
in the realization of the outcomes, our techniques will
still achieve \emph{observable} calibration from samples.
Recall that in observable calibration, we compare
the value $v$ output by the predictor on a category
$S_v$ to the empirical average of outcomes over the
category $\sum_{i \in S_v}o_i$.  
If we only need to guarantee closeness with
respect to these observable outcomes,
we do not need to ensure that sums over outcomes
will be concentrated around their expectation
(i.e. sums of the underlying true probabilities).

\subsection{Learning model}
In Section~\ref{sec:learning}, we present algorithms to learn
predictors satisfying accuracy-in-expectation and
calibration.  The algorithms
can be viewed as statistical query algorithms \cite{kearns}.
Specifically,
the algorithms only require access to approximate
statistical queries of the following form.

\begin{definition}[Statistical Query \cite{kearns}]
For a subset of the universe $S \subseteq \X$,
let $p_S^* = \sum_{i \in S} p_i^*$.
For $\tau \in [0,1]$, a \emph{statistical query with
tolerance $\tau$} returns some $\tilde{p}(S)$ satisfying
$$p_S^* - \tau N \le \tilde{p}(S) \le p_S^* + \tau N.$$
\end{definition}
Note that this query model guarantees \emph{absolute}
additive error $\tau N$.  As discussed above, our notion
of calibration with respect to $S$ asks for \emph{relative}
additive error; however, if we know a lower bound on
$\card{S} \ge \gamma N$, then, asking a statistical
query with tolerance $\tau = \alpha \gamma$ will guarantee
relative error $\alpha$ on $S$.

In addition to giving algorithms that work in this statistical
learning framework, we also address the question of learning
a calibrated predictor from a set of samples of outcomes.
Formally, we define the access to sampled outcomes as follows.
\begin{definition}[Random sample]
For $p^* \in [0,1]^N$, a \emph{random sample}
returns an individual-outcome pair $(i,o_i) \in \X\times \set{0,1}$,
where $i \sim \X$ is drawn uniformly at random and
$o_i$ is sampled according to the Bernoulli distribution
with parameter $p_i^*$.
\end{definition}
We say an algorithm learns a predictor from samples
if its only access to the true parameters $p^*$ is
through random samples of this form.
It's easy to see that we can always implement
a statistical query algorithm with access to enough random
samples -- for every query, we could sample a fresh set of
outcomes to estimate the statistic accurately.
Our goal will be to avoid resampling in order to
prevent a blow-up in the sample complexity.

\subsection{Multicalibration}
We introduce the notion of \emph{multicalibration}, which
requires calibration to hold simultaneously on subsets
of the population.  We will show that
multicalibration not only guarantees fairness
across protected populations, but also helps us uncover
more accurate predictions.  To motivate multicalibration
further, consider the following toy example:  suppose
$p^*$ is such that there is some population $S$ (possibly,
a traditionally protected group) and a subpopulation
$S' \subseteq S$ with $\card{S'} = \card{S}/2$,
where for every $i \in S'$,
$p_i^* = 1$, and for every $i \in S \setminus S'$, $p_i^* = 0$.
The predictor $x$ that predicts $x_i = 1/2$ for all $i \in S$ is calibrated on the population $S$, but clearly is suboptimal.
Further, if $S'$ was identifiable in advance,
then this predictor is arguably
discriminatory -- there are two clearly identifiable
groups within $S$, but we are treating them the same way.
If, however, we insist on calibration with respect to
$S'$ in addition to $S$, then the predictor will be required to
output accurate predictions for each group.
Earlier approaches to using calibration to achieve fairness,
as introduced in \cite{kleinberg2016inherent},
would prevent this form
of discrimination for subsets $S$ that are identified
as a protected group (defined, for example, by race),
but not for subpopulations of these groups -- even if the
subpopulations could be easily distinguished as outstanding.

For a collection of subsets $\C$,
we say that a predictor is $\alpha$-multicalibrated on $\C$ if it
is $\alpha$-calibrated simultaneously on all $S \in \C$.
\begin{definition}[$\alpha$-multicalibration]
Let $\C \subseteq 2^{\X}$ be a collection of subsets of $\X$
and $\alpha \in [0,1]$. A predictor
$x$ is $\alpha$-multicalibrated on $\C$ if for all $S \in \C$,
$x$ is $\alpha$-calibrated with respect to $S$.
\end{definition}
We also define the corresponding definition for the weaker
notion of $\alpha$-AE.
\begin{definition}[$\alpha$-multi-AE]
Let $\C \subseteq 2^{\X}$ be a collection of subsets of $\X$
and $\alpha \in [0,1]$. A predictor
$x$ is $\alpha$-multi-AE on $\C$ if for all $S \in \C$,
$x$ is $\alpha$-AE with respect to $S$.
\end{definition}

\paragraph{Discretization}
Even though $\alpha$-calibration is a meaningful definition
if we allow for arbitrary predictions $x_i \in [0,1]$, when
designing algorithms to learn calibrated predictors,
it will be useful to maintain some discretization
on the values $v \in [0,1]$.
Formally, we will use the following definition.
\begin{definition}[$\lambda$-discretization]
\label{def:lambda}
Let $\lambda > 0$.
The \emph{$\lambda$-discretization} of $[0,1]$, denoted by
$\Lambda[0,1] = \set{\frac{\lambda}{2},\frac{3\lambda}{2},\hdots,
1-\frac{\lambda}{2}}$,
is the set of $1/\lambda$ evenly spaced real
values over $[0,1]$.  For $v \in \Lambda[0,1]$,
let $$\lambda(v) = [v-\lambda/2,v+\lambda/2)$$
be the $\lambda$-interval centered around $v$
(except for the final interval, which will be $[1-\lambda,1]$).
\end{definition}

At times, it will be convenient to work with a more technical
variant of multicalibration, which implies
$\alpha$-multicalibration.
In particular, this definition will allow us to work
with an explicit discretization of the values $v \in [0,1]$.
Throughout, for a predictor $x$, we refer to the ``categories"
$S_v(x) = \set{i:x_i \in \lambda(v)} \cap S$ for all
$S \in \C$ and $v \in \Lambda[0,1]$.
\begin{definition}[$(\alpha,\lambda)$-multicalibration]
\label{def:al-mult}
Let $\C \subseteq 2^{\X}$ be a collection of subsets
of $\X$.
For any $\alpha, \lambda > 0$,
a predictor $x$ is \emph{$(\alpha,\lambda)$-multicalibrated} on $\C$
if for all $S \in \C$, $v \in \Lambda[0,1]$, and all
categories $S_v(x)$ such that
$\card{S_v(x)} \ge \alpha\lambda \card{S}$, we have
$$ \card{\sum_{i \in S_v(x)} x_i - p_i^*} \le \alpha \card{S_v(x)}.$$
\end{definition}
We claim that if learn a predictor that satisfies
$(\alpha,\lambda)$-multicalibration, we can easily transform
this predictor into one that satisfies our earlier notion
of $\alpha$-multicalibration.  In particular, let
$x^\lambda$ be the $\lambda$-discretization of a predictor
$x$ if for all $i \in S_v(x)$,
$x^\lambda_i = \E_{i \sim S_v(x)} [x_i]$.
\begin{claim}
\label{claim:lambda}
For $\alpha,\lambda > 0$,
suppose $\C \subseteq 2^{\X}$ is a collection of subsets of $\X$.
If $x$ is $(\alpha,\lambda)$-multicalibrated on $\C$,
then $x^\lambda$ is $(\alpha+\lambda)$-multicalibrated on $\C$.
\end{claim}
\begin{proof}
Consider the categories $S_v(x)$
where $\card{S_v(x)} < \alpha\lambda\card{S}$.
By the $\lambda$-discretization, there are at most
$1/\lambda$ such categories, so the cardinality of their
union is at most $(1/\lambda)\alpha\lambda \card{S}
= \alpha \card{S}$.  Thus, for each $S \in \C$, there is
a subset $S' \subseteq S$ with $\card{S'} \ge (1-\alpha)\card{S}$
where for all $v \in \Lambda[0,1]$,
$$\card{\E_{i \sim S_v(x) \cap S'}[x_i - p_i^*]} \le \alpha.$$
Further, $\lambda$-discretization will ``move" the values of
$x_i$ by at most $\lambda$, so overall,
$x^\lambda$ will be $(\alpha+\lambda)$-calibrated.
\end{proof}

Typically, we will imagine $\lambda = \Theta(\alpha)$,
but our results hold for any $\lambda \in (0,1]$.
Choosing a smaller $\lambda$ will allow the predictor
to be more expressive, but will also increase the
running time and sample complexity.  Choosing a larger
$\lambda$ leads to a decay in the calibration guarantees.

\paragraph{Representing subsets of individuals}
When representing collections of subsets, we will
assume that the subsets are represented implicitly.
In particular, we will assume that $S \in \C$ is given as
a circuit $c_S:\X \to \set{0,1}$, where $S = c_S^{-1}(1)$;
that is, for $i \in \X$,
$c_S(i) = 1$ if and only if $i \in S$.
Using this implicit
representation serves two purposes.  First, in many
cases, we may want to calibrate on a collection of
subsets over a large universe; in these cases, assuming an
explicit representation of each set is unreasonable.
Second, associating a set $S$ with a circuit that computes
membership in $S$ allows us to quantify the complexity of the sets
in $\C$.  In particular, it seems natural to apply
multicalibration to guarantee calibration with respect
to a collection of \emph{efficiently-identifiable} subsets (say,
subsets defined by conjunctions of four attributes, or
any other simple circuit class).
It seems comparatively
unreasonable to require calibration on, say, a random
subsets, each of which would require $\Omega(\card{\X})$
bits to describe.

\section{Learning $\alpha$-multicalibrated predictors}
\label{sec:learning}

In this section, we prove Theorem~\ref{result:alg}; that is,
we provide an algorithm for efficiently learning
$\alpha$-multicalibrated predictors.
The algorithms we describe are iterative and fit into the powerful
online optimization framework \cite{shaiSurvey,hazanSurvey}
as well as the statistical query framework \cite{kearns}.
For completeness, in Section~\ref{sec:AE}, we describe an algorithm that
solves the simpler task of learning an $\alpha$-multi-AE
predictor, as a warm-up to introduce the main ideas;
this algorithm was originally discovered by \cite{trevisan2009regularity}.
In Section~\ref{sec:alg}, we describe the algorithm
for learning $\alpha$-multicalibrated predictors in full.
Then, in Section~\ref{sec:privacy},
we will give a nontrivial implementation of the
statistical query oracle that will imply nontrivial
upper bounds on the running time and sample complexity
of the learning algorithm.
This implementation borrows ideas from the literature on
differentially private query release and optimization
\cite{HardtR10,ullman,privateLPs}.
We conclude in Section~\ref{sec:ckt} with the observation that our
algorithm also has implications for the circuit complexity
of calibrated predictors: for any collection of sets $\C$,
there is an $\alpha$-calibrated predictor whose circuit
complexity is a small factor larger than the circuit
complexity required to describe the sets in $\C$.
This establishes Theorem~\ref{result:ckt}.

\subsection{$\alpha$-multi-AE predictors}
\label{sec:AE}
We begin our discussion with a simpler statistical query
algorithm for learning an $\alpha$-multi-AE predictor.
This algorithm serves as a warm-up for the subsequent
algorithm for learning an $\alpha$-multicalibrated
predictor.

\begin{figure}[h]
{\bf Algorithm~\ref{alg:AE}} -- Learning
an $\alpha$-multi-AE predictor on $\C$\label{alg:AE}

\fbox{\parbox{\textwidth}{

Let $\alpha,\gamma > 0$ and let
$\C \subseteq 2^{\X}$ be such that for all
$S \in \C$, $\card{S} \ge \gamma N$.

For $S \subseteq \X$,
let $\tilde{p}(S)$ be the output of a statistical query
with tolerance $\tau < \alpha\gamma/4$.
\begin{itemize}
\item Initialize:
\\ $\circ$~Let $x = (1/2,\hdots,1/2) \in [0,1]^N$
\item Repeat:
\\$\circ$~For each $S \in \C$:
\begin{itemize}
\item Let $\Delta_S = \tilde{p}(S) - \sum_{i \in S}x_i$
\item If $\card{\Delta_S} > \alpha\card{S} - \tau N$:
\\update $x_i \gets x_i + \frac{\Delta_S}{\card{S}}$ for all $i \in S$
(projecting $x_i$ onto $[0,1]$ if necessary)
\end{itemize}
$\circ$~If no $S \in \C$ updated: exit and output $x$
\end{itemize}
}
}
\end{figure}

Algorithm~\ref{alg:AE} describes
an iterative statistical query procedure for learning
a $\alpha$-multi-AE predictor on $\C$.
Note that the problem of finding an $\alpha$-multi-AE predictor
for some collection of sets $\C$
can be written as a linear program; the algorithm presented
can be viewed as an instance of projected subgradient descent
(see \cite{hazanSurvey}).  The algorithm iteratively
updates a predictor $x$ until it cannot find a set
$S \in \C$ where the current estimate deviates
significantly from the value reported by the statistical
query.  We claim that if no set violates this condition,
then $x$ is $\alpha$-multi-AE on $\C$.

\begin{claim}
If Algorithm~\ref{alg:AE} outputs a predictor $x$,
then $x$ is $\alpha$-multi-AE on $\C$.
\end{claim}
\begin{proof}
Let $p_S = \sum_{i \in S}p_i^*$ and $x_S = \sum_{i \in S}x_i$.
By the assumed tolerance of the statistical queries $\tilde{p}(S)$,
we know that the queries are close to the $p_S$.
Specifically, we know $\card{\tilde{p}(S) - p_S^*} \le \tau N$
for some $\tau < \alpha \gamma$.
By the termination condition and the triangle inequality, for all $S \in \C$ we get
the estimate $\card{p_S^* - x_S} \le \card{\tilde{p}(S) - x_S} + \tau N
\le \alpha \card{S}$; thus $x$ is $\alpha$-AE on $\C$.
\end{proof}
Thus, to show the correctness of the algorithm, it
remains to show that the algorithm will, in fact,
terminate; we show the algorithm can make
at most $O(1/\alpha^2\gamma)$ updates.
\begin{lemma}
Suppose $\alpha,\gamma > 0$
and $\C \subseteq 2^{\X}$ such that
for all $S \in \C$, $\card{S} \ge \gamma N$. 
Let $\tau = \alpha \gamma / 4$.  Then
Algorithm~\ref{alg:AE} makes $O(1/\alpha^2\gamma)$
updates to $x$ before terminating.
\label{lem:AE}
\end{lemma}
\begin{proof}
We use a potential argument,
tracking the progress the algorithm makes on each update
in terms of the $\ell_2^2$ distance between our learned
predictor $x$ and the true predictions $p^*$.  Let $x'$
be the predictor after updating $x$ on set $S$ and
let $\pi:\R \to [0,1]$ denote projection onto $[0,1]$.
We use the fact that the $\ell_2^2$ can only decrease
under this projection.
For notational convenience, let $\delta_S = \dfrac{\Delta_S}{\card{S}}
= \dfrac{1}{\card{S}}(\tilde{p}(S) - \sum_{i \in S}x_i)$. We have
\begin{align*}
\norm{p^* - x}^2 - \norm{p^* - x'}^2
&= \sum_{i \in S}(p_i^*-x_i)^2 -
\sum_{i \in S}(p_i^*-\pi(x_i + \delta_S))^2\\
&\ge \sum_{i \in S}((p_i^*-x_i)^2-(p_i^* - (x_i + \delta_S))^2)\\
&= \sum_{i \in S}(2(p_i^* - x_i)\delta_S - \delta_S^2)\\
&= \left(2\delta_S\sum_{i \in S}(p_i^* - x_i)\right) - \delta_S^2\card{S}\\
&\ge 2\delta_S\left(\delta_S \card{S} -
{\rm sgn}(\delta_S)\tau N\right) - \delta_S^2\card{S}\\
&\ge \delta_S^2\card{S} - 2\card{\delta_S}\tau N.
\end{align*}
By setting $\tau = \alpha \gamma/4$ and by the bound $\card{\Delta_S} \ge \alpha \card{S} - \tau N \ge 3\alpha\card{S}/4$,
the final quantity is at least $\Omega(\alpha^2\card{S})$. We also have
\begin{align*}
\delta_S^2\card{S} - 2\card{\delta_S}\tau N
& \ge \left( \frac{3\alpha}{4} \right)^2 \card{S}
- 2 \left(\frac{3\alpha}{4}\right)\left(\frac{\alpha\gamma}{4}\right) N\\
&= \frac{3\alpha^2}{16}\card{S}.
\end{align*}
The $\ell_2^2$ distance
between $p^*$ and any other predictor (in particular, our
initial choice for $x$) is upper-bounded by $N$.  Thus, given
that all $S \in \C$ have $\card{S} \ge \gamma N$,
we make at least $\Omega(\alpha^2\gamma N)$ progress in potential
at each update, so the lemma follows.
\end{proof}
In combination, these statements show the correctness of
Algorithm~\ref{alg:AE} and imply an upper bound on the number
of statistical queries necessary.
\begin{theorem}
For $\alpha,\gamma > 0$ and for any $\C \subseteq 2^{\X}$
satisfying $\card{S} \ge \gamma N$ for all $S \in \C$,
there is a statistical query algorithm with tolerance
$\tau = \alpha\gamma/4$ that learns
a $\alpha$-multi-AE predictor on $\C$ in $O(\card{\C}/\alpha^2\gamma)$
queries.\label{thm:AE}
\end{theorem}
Recall that, trivially, we could implement this statistical
query algorithm from random samples by resampling for
every query; however, in this case, we can easily improve
the sample complexity exponentially over the trivial solution.
Specifically, the queries we make are \emph{non-adaptive}
because, up front, we know a fixed collection of subsets
whose expectation we might query. To guarantee accurate
expectations on this fixed collection, we only
need enough samples to guarantee that the sample is inaccurate
on a fixed subset with very small probability, and then union
bound over all $\card{\C}$ subsets.  Appealing to a standard
generalization argument \cite{kearnsvazirani},
we can show the following theorem.

\begin{corollary}
Suppose $\alpha,\gamma,\xi > 0$
and $\C \subseteq 2^{\X}$ is such that
for all $S \in \C$, $\card{S} \ge \gamma N$.  Then
there is an algorithm that learns an $\alpha$-multi-AE
predictor on $\C$ with probability at least $1-\xi$
from $n = \tilde{O}\left(\dfrac{\log(\card{\C}/\xi)}{\alpha^2\gamma}\right)$
samples.\label{cor:AE}
\end{corollary}
Note that the $\gamma$ dependence in the sample
complexity is only $1/\gamma$.  Naively, applying the guarantees
of the statistical query oracle, we would obtain a
$1/\gamma^2$ dependence.  To achieve this bound,
we note that because calibration requires relative
error, we can be more judicious with our use of samples.
We will exploit this observation subsequently to obtain
improvements to the sample complexity for learning
$\alpha$-multicalibrated predictors.
\begin{proof}
To obtain the claimed sample
complexity bound, we observe that in Algorithm~\ref{alg:AE},
we only use the statistical query oracle to guarantee bounds
on the relative error of each query -- not absolute error.
In particular, for $S \subseteq \X$ with $\card{S} \ge \gamma N$,
let $\bar{p}_S = \frac{1}{\card{S}}\sum_{i \in S}p_i^*$.
To run Algorithm~\ref{alg:AE},
we need only implement an oracle $\hat{p}(S)$ satisfying
$$\bar{p}_S - \tau \le \hat{p}(S) \le \bar{p}_S + \tau.$$
By Chernoff bounds, for a fixed set $S$ of cardinality at
least $\gamma N$, if we take $\tilde{O}(t/\gamma \alpha^2)$
independent samples, the probability that the estimate of
$\bar{p}_S$ differs by more than $\alpha \card{S}$ is
at most $e^{-\Omega(t)\log(1/\gamma\alpha)}$.  Taking $t = c(\log\card{\C} + \log(1/\xi))$
for an appropriate constant $c$,
a union bound implies that for each of $\poly(1/\gamma,1/\alpha)$
iterations, the probability that the estimate of
every set $S \in \C$ is $\alpha$-accurate is at least
$1-\xi$.
\end{proof}

\subsection{$\alpha$-Multicalibrated predictors}
\label{sec:alg}

Next, we present the full algorithm for learning
$\alpha$-multicalibrated predictors on $\C$.  The algorithm
is based on Algorithm~\ref{alg:AE} but differs
in a few key ways.  First, instead of updating
the predictions on entire sets $S \in \C$ whose overall
expectation is wrong, we update the predictions on
uncalibrated categories $S_v = S\cap \set{i: x_i = v}$.
This is a simple change to the
algorithm in the statistical query model; however,
when we wish to implement this statistical query oracle
from a finite sample, we need to be more careful.

In particular, the categories of the predictor we learn
are not fixed \emph{a priori}, so our queries will be
selected \emph{adaptively} based on the results of earlier
statistical queries.  Stated another way,
we cannot simply union bound against the collection
of sets on which we wish to be calibrated.
The most naive approach to bounding the sampling complexity
would be as follows: at each iteration take a fresh
sample large enough to guarantee $\omega N$ absolute
error, for $\omega = \alpha \gamma$.  Following our analysis
below for this naive strategy, it's easy to see that
this approach would result in
$\Omega(1/\alpha^4\gamma^4)$ iterations
and sample complexity $n \ge \Omega(1/\alpha^6\gamma^6)$.
We will improve on this approach by polynomial factors
achieving a bound of at most $O(1/\alpha^4\gamma)$ iterations
with $O(1/\alpha^{5/2}\gamma^{3/2})$ samples.

To achieve these improvements, we combine two
ideas.  As before, we will leverage
the observation that calibration only requires
relative error (as in Corollary~\ref{cor:AE}), and
thus, in principle should require fewer samples.
Additionally,
to avoid naively resampling but still guarantee
good generalization from a small
sample, we interact with the sample through a mechanism
which we call a \emph{guess-and-check} statistical query
(similar in spirit to mechanisms proposed in
\cite{HardtR10,blum2015ladder,gupta2012iterative}).
We show how to implement this mechanism in a manner that
guarantees generalization on the unseen data
even after asking many adaptively chosen statistical
queries.  We defer our discussion of privacy to Section~\ref{sec:privacy}.

\paragraph{Details of the algorithm}

Next, we give an iterative procedure to learn a
$(\alpha,\lambda)$-multicalibrated predictor on $\C$ as described in
Algorithm~\ref{alg:calibrated}.
The procedure is similar to Algorithm~\ref{alg:AE},
but deliberately interacts with its statistical
queries through a so-called guess-and-check oracle.
In particular, each time the algorithm needs to know the
value of a statistical query on a set $S$, rather than
asking the query directly, we require that the algorithm
submit its current guess
$x_S = \frac{1}{\card{S}}\sum_{i \in S} x_i$ to the
oracle, as well as an acceptable ``window" $\omega \in [0,1]$.
Intuitively,
if the algorithm's guess is far from the window centered around
the true expectation, then the oracle will
respond with the answer to a statistical query with tolerance $\omega$.
If, however, the guess is sufficiently close to the
true value, then the oracle responds with $\checkmark$ to
indicate that the current guess is close to the expectation,
without revealing another answer.
\begin{definition}[Guess-and-check oracle] Let
$\tilde{q}:2^{\X}\times[0,1]\times[0,1] \to [0,1] \cup \set{\checkmark}$. $\tilde{q}$
is a \emph{guess-and-check oracle} with window $\omega_0$
if for $S \subseteq \X$ with
$p_S = \sum_{i \in S}p_i^*$, $v \in [0,1]$, and any
$\omega \ge \omega_0$,
the response to $\tilde{q}(S,v,\omega)$
satisfies the following conditions:
\begin{itemize}
\item if $\big\vert p_S - \card{S}v\big\vert < 2\omega N$,
then $\tilde{q}(S,v,\omega) = \checkmark$
\item if $\big\vert p_S - \card{S}v\big\vert > 4\omega N$,
then $\tilde{q}(S,v,\omega) \in [0,1]$
\item if $\tilde{q}(S,v) \neq \checkmark$, then
$$ p_S - \omega N \le \tilde{q}(S,v,\omega) \card{S} \le p_S + \omega N.$$
\end{itemize}
\end{definition}
Note that if the guess is such that
$\big\vert p_S - \card{S}v\big\vert \in [2\omega N,4\omega N]$,
the the oracle may respond with some $\omega$-accurate
$r \in [0,1]$ or with $\checkmark$.
Of course, if we have a lower bound $\omega_0$ on the
window over a sequence of guess-and-check queries,
we can implement the queries given access to
a statistical query oracle with tolerance $\tau \le \omega_0$;
it is also clear that a statistical oracle with tolerance $\tau$
can be implemented with access to a guess-and-check oracle
with window $\tau/4$.  The advantage of using this
guess-and-check framework is that it can be implemented
in a differentially private manner.
This will in turn allow us to give an algorithm for learning
$\alpha$-multicalibrated predictors from a small number of
samples that generalizes well.

\begin{figure}[h]
{\bf Algorithm~\ref{alg:calibrated}} -- Learning
a $(\alpha,\lambda)$-calibrated predictor on $\C$

\fbox{\parbox{\textwidth}{

Let $\alpha,\lambda > 0$ and let
$\C \subseteq 2^{\X}$ be such that for all $S \in \C$,
$\card{S} \ge \gamma N$.

For $S \subseteq \X$ and $v \in [0,1]$,
let $\tilde{q}(\cdot,\cdot,\cdot)$ be a guess-and-check oracle.
\begin{itemize}
\item Initialize:
\\ $\circ$~Let $x = (1/2,\hdots,1/2) \in [0,1]^N$
\item Repeat:
\\$\circ$~For each $S \in \C$, $v \in \Lambda[0,1]$, for each
$S_v = S \cap \set{i : x_i = \lambda(v)}$ such that
$\card{S_v} = \beta N \ge \alpha\lambda \card{S}$
\begin{itemize}
\item Let $\bar{v} = \frac{1}{\card{S_v}}\sum_{i \in S_v}x_i$
\item Let $r = \tilde{q}(S_v,\bar{v},\alpha\beta/4)$
\item If $r \neq \checkmark$:\\
update $x_i \gets x_i + (r-\bar{v})$ for all $i \in S_v$
(projecting $x_i$ onto $[0,1]$ if necessary)
\end{itemize}
$\circ$~If no $S_v$ updated, exit
\item For $v \in \Lambda[0,1]$:\\
$\circ$~Let $\bar{v} = \sum_{i \in \lambda(v)}x_i$\\
$\circ$~For $i \in \lambda(v)$: $x_i \gets \bar{v}$
\item Output $x$
\end{itemize}
}
}
\label{alg:calibrated}
\end{figure}

Algorithm~\ref{alg:calibrated} runs through each possible
category $S_v$ and if $S_v$ is large enough, queries
the oracle.  The algorithm continues searching for uncalibrated
categoires until $x$'s guesses on all sufficiently large categories
receive $\checkmark$.
By the definition of the guess-and-check oracle,
if for some category $S_v$ where $\card{S_v} = \beta N$
the query returns $\checkmark$, then $\bar{v}$ is at most
$4\cdot (\alpha \beta N/4) = \alpha \card{S_v}$ far from
the true value $\frac{1}{S_v}\sum_{i \in S_v}p_i^*$.
Thus, by the stopping condition of the loop,
the predictor where all $i \in \lambda(v)$ receive
$x_i = \bar{v}$ will be $\alpha$-calibrated on every large
category.
Finally, the algorithm updates $x$ to be $\lambda$-discretized,
so by Claim~\ref{claim:lambda},
$x$ will be $(\alpha+\lambda)$-calibrated.
Further, the number of updates necessary to terminate is bounded.
\begin{lemma}
\label{lem:calibrated}
Suppose $\alpha,\lambda > 0$ and $\C \subseteq 2^{\X}$ where
for all $S \in \C$, $\card{S} \ge \gamma N$.
Algorithm~\ref{alg:calibrated} returns $x$ after receiving
at most $O(1/\alpha^3\lambda\gamma)$ guess-and-check responses
where $r \in [0,1]$ and at most
$O(\card{C}/\alpha^4\lambda\gamma)$ responses $r = \checkmark$.
\end{lemma}
\begin{proof}
For some non-$\checkmark$ response
on $S_v = \set{i: x_i \in \lambda(v)} \cap S$,
by the properties of the guess-and-check oracle, we
can lower bound the update step size.  Recall, we only
query on sets wherer $\card{S_v} = \beta N \ge \alpha \card{S}$
with a window of $\omega = \alpha \beta/4$.
\begin{align*}
\card{\sum_{i \in S_v}p_i^* - x_i} &=
\card{\sum_{i \in S_v}p_i^* - \bar{v}}\\
&\ge 2\left(\alpha\beta/4\right) N\\
&= \alpha \card{S_v}/2.
\end{align*}
Letting $\delta_v = r - \bar{v}$.
We can measure progress in the same way as in
Lemma~\ref{lem:AE}.
\begin{align*}
\norm{p^* - x}^2 - \norm{p^* - x'}^2
&= \sum_{i \in S_v}(p_i^*-x_i)^2 -
\sum_{i \in S_v}(p_i^*-\pi(x_i + \delta_v))^2\\
&\ge \sum_{i \in S_v}((p_i^*-x_i)^2-(p_i^* - (x_i + \delta_v))^2)\\
&= \sum_{i \in S_v}(2(p_i^* - x_i)\delta_v - \delta_v^2)\\
&= \left(2\delta_v\sum_{i \in S_v}(p_i^* - x_i)\right) - \delta_v^2\card{S_v}
\end{align*}

Let $\nu = \frac{1}{\card{S_v}}\sum_{i \in S_v}(p_i^* - x_i)$.
By the properties
of the guess-and-check oracle, we can rewrite $\delta_v$ as
$\nu-\eta$ for some $\eta \in [-\omega/\beta,\omega/\beta]$.
This gives us a lower bound on the progress as follows.
$$ \left(2(\nu-\eta)\nu - (\nu-\eta)^2\right)\card{S_v}
= \left(\nu^2 + \nu\eta - (\eta)^2\right)\card{S_v}$$
This concave function in $\eta$ is minimized at an extreme
value for $\eta$ (depending on the sign of $\nu$).
Noting that $\card{\nu} \ge \alpha/2$ and
$\card{\eta} \le \omega/\beta = \alpha/4$,
we can lower bound our progress by $(\alpha/4)^2 \card{S_v}
= \alpha^2\beta N/16 = \alpha^3\lambda\gamma N/16$.
As $\norm{p^*}^2 \le N$,
we make at most $O(1/\alpha^3\lambda\gamma)$ updates
upper bounding the number of non-$\checkmark$ responses.
By working with a $\lambda$-discretization,
there are at most $\card{C}/\lambda$ categories to consider
in every phase, so we receive at most $O(\card{C}/\alpha^3\lambda^2\gamma)$
$\checkmark$ responses.
\end{proof}

Thus, we conclude the following theorem.
\begin{theorem}
For $\alpha,\lambda > 0$ and $\C \subseteq 2^{\X}$ where for
all $S \in \C$, $\card{S} \ge \gamma N$, there
is a statistical query algorithm that learns a
$(\alpha,\lambda)$-multicalibrated predictor with respect to $\C$ in
$O(\card{\C}/\alpha^3\lambda^2\gamma)$ queries.
\end{theorem}
Again, note that our output is, in fact, $(\alpha+\lambda)$-multicalibrated,
so taking $\lambda = \alpha$, we obtain a $(2\alpha)$-multicalibrated
predictor in $O(\card{\C}/\alpha^5\gamma)$ queries.

\subsection{Answering guess-and-check queries from a random sample}
\label{sec:privacy}
Next, we argue that we can implement a guess-and-check oracle
from a set of random samples in a manner that guarantees good generalization.  
This, in turn, allows us to translate our statistical query
algorithm for learning an $(\alpha,\lambda)$-multicalibrated predictor
with respect to $\C$
into an algorithm that learns from samples.
As mentioned in the beginning of Section~\ref{sec:learning},
naively, we could resample for every update
the algorithm makes to the predictor.
Suppose that $\C$ is such that for all $S \in \C$,
$\card{S} \ge \gamma N$; let $\beta = \alpha\lambda\gamma$.
Using our tighter analysis of Algorithm~\ref{alg:calibrated},
we could take
$n = \tilde{O}(\log(\card{\C})/\alpha^2\beta)$
samples per update to guarantee generalization,
resulting in an overall sample complexity of
$\tilde{O}(\log(\card{\C})/\alpha^4\beta^2)$.
We show how to improve upon this approach further.
In particular, we argue that there is a differentially
private algorithm that
can answer the entire sequence of guess-and-check
queries accurately.  Appealing to known
connections between differential privacy and adaptive
data analysis \cite{dwork2015preserving,dwork2015generalization,bnsssu15,dwork2015reusable},
this will guarantee that our calibration
algorithm generalizes given a set of
$\tilde{O}(\log(\card{\C})/\alpha^{5/2}\beta^{3/2})$
random samples.

Algorithm~\ref{alg:calibrated} only interacts with the
sample through the guess-and-check oracle.  Thus, to give
a differentially private implementation of the algorithm,
it suffices to give a differentially private implementation
of the guess-and-check oracle \cite{dwork2014algorithmic}.

Consider the sequence of queries that Algorithm~\ref{alg:calibrated}
makes to the guess-and-check oracle.  We say the sequence
$\langle (S_1,v_1,\omega_1),\hdots,(S_k,v_k,\omega_k) \rangle$
is a $(k,m)$-sequence of guess-and-check queries if, over the course
of the $k$ queries, the response to
at most $m$ of the queries is some $r \in [0,1]$, and the
responses to the remaining queries are all $\checkmark$.
We will assume that we know
a lower bound on the minimum window
$\omega = \min_{j \in [k]} \omega_j$ over all of the queries.
We say that some algorithm $\A$ responds to a guess-and-check
query $(S,v,\omega)$ according to a random sample $X$ if its
response satisfies the guess-and-check properties with
$\sum_{i \in S} p_i^*$ replaced by
its empirical estimate on $X$,
$$\hat{p}_S(X) = \frac{\card{S}}{\card{S \cap X}}\sum_{i \in S \cap X} o_i.$$
Responding to such a sequence in a differentially private
manner can be achieved using techniques from the
private multiplicative weights mechanism.
\begin{theorem}[\cite{HardtR10}~]
\label{thm:private}
Suppose $\eps,\delta,\omega,\xi > 0$ and suppose
$X \sim (\X \times \set{0,1})^n$ is a set of $n$ random samples.
Then there exists an $(\eps,\delta)$-differentially private algorithm $\A$
that responds to any $(k,m)$-sequence of guess-and-check queries
with minimum window $\omega$ according to $X$ provided
$$n = \Omega\left(\sqrt{\frac{\log (k/\xi) \cdot m \cdot \log (1/\delta)}{\eps \cdot \omega^2}}\right)$$
with probability at least $1-\xi$ over the randomness of
$\A$.
\end{theorem}

Using this differentially private algorithm, we can apply
generalization bounds based on privacy
developed in \cite{dwork2015preserving,bnsssu15,dwork2015generalization,dwork2015reusable}
to show that, with a modest increase in sample complexity,
we can respond to all $k$ guess-and-check queries.
\begin{theorem}
\label{thm:oracle}
Let $s_k = \langle (S_1,v_1,\hat{\omega}_1),\hdots,(S_k,v_k,\hat{\omega}_k)
\rangle$ be a $(k,m)$-sequence of guess-and-check queries such
that for all $j \in [k]$, $\card{S_j} = \beta_j N \ge \beta N$
and $\hat{\omega}_j = \Omega(\alpha \beta_j)$.
Then there is an algorithm $\A$ that, given $n$ random samples
$X \sim (\X\times \set{0,1})^n$,
responds to $s_k$ such that for all $j \in [k]$, the response
$\A(S_j,v_j,\hat{\omega}_j; X)$ satisfies the guess-and-check
properties with window $\omega_j = \alpha \beta_j$ provided
$$n = \Omega\left(\frac{\log(\card{\C}/\alpha\beta\xi)}
{\alpha^{5/2}\cdot \beta^{3/2}}\right)$$
with probability at least $1-\xi$ over the randomness of
$\A$ and the draw of $X$.
\end{theorem}

This theorem implies that, asymptotically, we can answer
the $k$ adaptively chosen guess-and-check queries with only a
$\sqrt{1/\alpha\beta}$ factor increase in the sample complexity
compared to if we knew the queries in advance.
Theorem~\ref{thm:oracle} follows from tailoring the
proof of the main ``transfer" theorem of
\cite{bnsssu15} (Theorem~3.4) specifically to the
requirements of our guess-and-check oracle and applying the
differentially private mechanism described in
Theorem~\ref{thm:private}.
Combining these theorems and Algorithm~\ref{alg:calibrated}
and the fact that $\beta = \alpha\lambda\gamma$,
we obtain an algorithm for learning $\alpha$-multicalibrated
predictors from random samples.
\begin{theorem}
\label{thm:calibrated}
Suppose $\alpha,\lambda,\gamma, \xi > 0$,
and $\C \subseteq 2^{\X}$ where for
all $S \in \C$, $\card{S} \ge \gamma N$.  Then
there is an algorithm that learns an $(\alpha,\lambda)$-multicalibrated
predictor with respect to $\C$ with probability at least $1-\xi$
from $n = O\left(\dfrac{\log(\card{\C}/\alpha\lambda\gamma\xi)}{\alpha^4\cdot \lambda^{3/2}\cdot \gamma^{3/2}}\right)$ samples.
\end{theorem}

\subsection{Runtime analysis of Algorithm~\ref{alg:calibrated}}
\label{sec:runtime}
Here, we present a high-level runtime analysis of
Algorithm~\ref{alg:calibrated} for learning an
$(\alpha,\lambda)$-calibrated predictor on $\C$.
In Lemma~\ref{lem:calibrated}, we claim an upper bound
of $O(\card{\C}/\alpha^3\lambda^2\gamma)$ on
the number of guess-and-check queries needed before
Algorithm~\ref{alg:calibrated} converges.
Here, we formally
argue that each of these queries can be implemented
in the random sample model without much overhead,
which upper-bounds the running time of the algorithm overall.
This upper bound is not immediate from our earlier analysis,
as the sets and our predictor are represented implicitly
as circuits.
\begin{claim}
Algorithm~\ref{alg:calibrated} runs in time
$O(\card{\C} \cdot t\cdot\poly(1/\alpha,1/\lambda,1/\gamma))$,
where $t$ is an upper bound on the time it takes to
evaluate set membership for $S \in \C$.
\end{claim}
\begin{proof}
As before, let $\beta = \alpha\lambda\gamma$.
First, for each $S \in \C$,
we need to evaluate $\card{S_v}$ for $S_v = \set{i : x_i \in \lambda(v)}\cap S$
for each of the $O(1/\lambda)$ values $v \in \Lambda[0,1]$.
We do this by sampling $i \sim \X$ and evaluating
whether $i \in S$, and if so, checking the current value
of $x_i$.  Each of the membership queries takes at most $t$
time and each evaluation of $x_i$ takes at most
$O(t/\alpha^2\beta)$ time by the same
argument as our upper bound on the circuit
size from Theorem~\ref{thm:ckt}.
After $\tilde{O}(1/\lambda\beta^2)$ samples, we will
be able to detect with constant probability which of the
$S_v$ have cardinality $\card{S_v} \ge \beta N$.
Further, if $\card{S_v}$ is large,
we can estimate $\bar{v}$
by evaluating the current predictor on samples
from $S_v$, by rejection sampling.
Similarly,
to answer the guess-and-check queries, we
will estimate the true empirical estimate
of the query based on samples from $S_v$ and
respond based on a noisy comparison between
the $\bar{v}$ and the estimate of $\sum_{i \in S_v} o_i$.
These estimates can all be computed in $\poly(1/\alpha,1/\beta)$.
Then, as discussed in the proof of Theorem~\ref{thm:ckt},
each update to the predictor can be implemented in
time proportional to the bit complexity of the arithmetic
computations, which is upper bounded by $t$.
Repeating this process for each $S \in \C$ gives
the upper bound of
$O(\card{\C}\cdot t \cdot \poly(1/\alpha,1/\lambda,1/\gamma))$.
Finally, applying the upper bound on the number of
guess-and-check queries from Lemma~\ref{lem:calibrated},
the claim follows.
\end{proof}

\subsection{The circuit complexity of multicalibrated predictors}
\label{sec:ckt}
As discussed in Section~\ref{sec:intro:contributions}, an interesting corollary of our
algorithm is a theorem about the complexity of representing
a multicalibrated predictor.  Indeed, from the definition
of multicalibration alone, it is not immediately clear that
there should be succinct descriptions of multicalibrated
predictors; after all, $\C$ could contain many sets.
We argue that the cardinality of $\C$ is not the operative
parameter in determining the circuit complexity of a
predictor $x$ that is multicalibrated on $\C$; instead
it is the circuit complexity necessary to describe sets
$S \in \C$, as well as the cardinality of the subsets in
$\C$, and the degree of approximation.

Leveraging Lemma~\ref{lem:calibrated}, we can see that
Algorithm~\ref{alg:calibrated} actually gives us a way to
build up a circuit that computes the mapping from individuals
to the probabilities of our learned multicalibrated predictor
$x$.  Suppose that for all sets $S \in \C$, set membership
can be determined by a circuit family of bounded complexity;
that is, for all $S \in \C$, there is some
$c_S$ with size at most $s$,
such that $c_S(i) = 1$ if and only if $i \in S$.
Then we can use this family of circuits to build a circuit
that implements $x$.  We assume that we maintain
real-valued numbers up to $b \ge \log(1/\alpha)$ bits of precision.
\begin{theorem}
\label{thm:ckt}
Suppose $\C \subseteq 2^{\X}$ is a collection of sets
where each $S \in \C$ can be implemented by a boolean
circuit $c_S$ and for all $S \in \C$,
the size of $c_S$ is $O(s)$.  Then there is a
predictor that is $\alpha$-multicalibrated on $\C$ implemented
by a circuit of size $O((s+b)/\alpha^4\gamma)$.
Further, Algorithm~\ref{alg:calibrated} can be used to learn
such a circuit.
\end{theorem}
\begin{proof}
We describe how to construct a circuit $f_x$ that,
on input $i$,
will output the prediction $x_i$ according to the
predictor learned by our algorithm.
We initialize $f_x$ to be the constant function
$f_x(i) = 1/2$ for all $i \in \X$.  Throughout, we will
update $f_x$ \emph{based on the current outputs of $f_x$}.

Consider an iteration of Algorithm~\ref{alg:calibrated}
where for some $S$ described by $c_S \in \C$,
we update $x$ based on a category $S_v = S \cap
\set{i:x_i \in \lambda(v)}$.
This occurs when the guess-and-check
query returns some $r = \tilde{q}(S_v,\bar{v},\omega) \in [0,1]$.
Our goal is to implement the update to $x$ (i.e. update
$f_x$), such that
for all $i \in S_v$, the new value $x_i = r$ and all other
values are unchanged.

We achieve this update by testing membership $i \in S$ and
separately testing if the current value $f_x(i) = v$; if both
tests pass, then we update the value output by $f_x(i)$ to be $r$. 
Specifically, we include a copy of $c_S$ and
hard-code $v$ and $\delta_v = r-\bar{v}$ into the circuit;
if $c_S(i) = 1$
and the current value of $f_x(i)$ is in $\lambda(v)$,
then we update $f_x(i)$ to add the hardcoded $\delta_v$ to its
current estimate of $x_i$; if either test fails, then $f_x(i)$
remains unchanged.  This logic can be implemented with
addition and subtraction circuits to a precision of
$\lambda$ with boolean circuits
of size $O(b)$.
We string these update circuits together,
one for each iteration.  Learning an
$(\alpha/2,\alpha/2)$-multicalibrated predictor
with Algoirthm~\ref{alg:calibrated} only requires
$O(\alpha^4\gamma)$ updates. By this upper bound,
we obtain an $O(\alpha^4\gamma)$ upper bound
on the resulting circuit size.
\end{proof}

\section{Multicalibration and weak agnostic learning}
\label{sec:weak}

Note that in the algorithm and analysis in Section~\ref{sec:learning},
we've assumed
nothing about the structure of the underlying $p^*$
or $\C$; the true probabilities could be adversarially chosen and
yet, our algorithm guarantees $\alpha$-multicalibration
on $\C$.  That said, the running time of the algorithm
depends linearly on $\card{\C}$.  As we imagine $\C$ to
be a large, rich class of subsets of $\X$, in many cases
linear depedence on $\card{\C}$ will be expensive.
Thus, we turn our attention to when we can exploit
structure within the collection of subsets $\C$ to speed up
the learning process.

The main running time bottleneck in the algorithms arises
from searching for some $S \in \C$ where calibration is
violated.  Without any assumptions about $\C$, we need to
loop over the collection; however, if we can find such a set
without looping over the entire collection of sets,
then we would improve the running time of the algorithm.
At a high level, we will show a connection between the
agnostic learnability of $\C$ and the ability to speed
up learning a multicalibrated predictor on $\C$.
Imagining the sets $S \in \C$ as boolean concepts,
we show that if it is possible to perform weak agnostic
learning over a class $\C$ efficiently (in the sense
of \cite{kalai2008agnostic,feldman2010distribution}),
then there is an efficient search algorithm to find
an update to the current predictor that will make progress towards
multicalibration.

While there are some classes for which we have weak agnostic
learners, in general, agnostic learning is considered a notoriously
challenging problem.  A natural question to ask is whether there
is a way to speed up learning a multicalibrated predictor
that does not involve agnostic learning.  We answer this
question in the negative.  Roughly, we show that for a concept class
$\C$,
we can use any predictor that is multicalibrated
on the large sets of $\C$ to a query for
distribution-specific weak agnostic learning on $\C$.
In this sense, the reduction to weak agnostic
learning is inherent; any efficient algorithm for multicalibration
gives rise to an algorithm for weak agnostic learning.

In all, these results show that weak agnostic learning on
a class $\C$ is equivalent to learning an $\alpha$-multicalibrated
predictor with respect to $\C = \set{S \in \C:\card{S} \ge \gamma N}$,
the large sets defined by $\C$, up to polynomial factors in
$1/\alpha,1/\gamma$ where $\rho$ and $\tau$ will be a function
of $\alpha$ and $\gamma$.

\subsection{Weak agnostic learning}

For this discussion, we think
of boolean concepts $c \in \C$ as $c: \X \to \set{-1,1}$.
We will overload the notions of
a concept class $\C$ of boolean functions
$c:\X \to \set{-1,1}$ and our collection
of subsets $\C \subseteq 2^\X$; in particular,
there is a natural bijection between concepts and
sets: a concept $c:\X \to \set{-1,1}$ defines a
set $S \subseteq 2^\X$ where $i \in S$ if $c(i) = 1$
and $i \not \in S$ if $c(i) = -1$.
We will connect
the problem of finding a set $S \in \C$ on which a predictor
$x$ violates calibration to the problem of learning
over the concept class $\C$ on a distribution $\D$.

For some distribution $\D$ supported on $\X$ and $x,y \in [-1,1]^N$,
let $\langle x,y \rangle_\D = \sum_{i \in \X} \D_i x_i y_i$.
This inner product measures the correlation between $x$ and $y$ in $[-1,1]^N$
over the discrete distribution $\D$.  Throughout our discussion,
we will focus on learning over the uniform distribution on
$\X$ and drop explicit reference to $\D$.
As per Remark~\ref{rem:dist}, this may be a rich distribution
over the features of individuals.

In our results, we will work with the \emph{distribution-specific}
weak agnostic learners of \cite{feldman2010distribution}\footnote{Often,
such learners are defined in terms
of their error rates rather than correlations;
the definitions are equivalent up to
factors of $2$ in $\rho$ and $\tau$.
Also, we will
always work with a hypothesis class $\H = [-1,1]^{\X}$ the
set of functions from $\X$ to $[-1,1]$, so we fix this
class in the definition.}.

\begin{definition}[Weak agnostic learner]
Let $\rho \ge \tau > 0$.
Let $\D$ be a distribution supported on $\X$.
A \emph{$(\rho,\tau)$-weak agnostic learner}
$\cal L$ for $\D$ solves the following promise problem:
given a collection of labeled samples $\set{(i,y_i)}$
where $i \sim \D$ and $y_i \in [-1,1]$, if there is some
$c \in \C$ such that $\langle c, y \rangle_\D > \rho$,
then $\cal L$ returns some $h:\X \to [-1,1]$
such that $\langle h,y \rangle_\D > \tau$.
\end{definition}
Intuitively, if there is a concept $c \in \C$ that
correlates nontrivially with the observed labels,
then the weak agnostic learner returns a hypothesis $h$ (not necessarily from
$\C$), that is also nontrivially correlated with the observed labels.
In particular, $\rho$ and $\tau$ are typically
taken to be $\rho = 1/p(d)$ and $\tau = 1/q(d)$ for
polynomials $p(d) \le q(d)$, where $d = \log(\card{\X})$.

\subsection{Multicalibration from weak agnostic learning}
\label{sec:wal2mult}

In this section, we show how we can use a weak agnostic learner
to solve the search problem
that arises at each iteration of Algorithm~\ref{alg:calibrated}:
namely, to find an update that will make progress towards mulitcalibration.
Formally, we show the following theorem.
\begin{theorem}[Formal statement of Theorem~\ref{result:boosting}]
\label{thm:weak}
Let $\rho,\tau > 0$ and $\C \subseteq 2^{\X}$ be some
concept class. If $\C$ admits a $(\rho,\tau)$-weak agnostic
learner that runs in time $T(\card{\C},\rho,\tau)$,
then there is an algorithm that learns a predictor that
is $(\alpha,\lambda)$-multicalibrated on
$\C' = \set{S \in \C: \card{S} \ge \gamma N}$ in time
$O(T(\card{\C},\rho,\tau)\cdot\poly(1/\alpha,1/\lambda,1/\gamma))$
as long as $\rho \le \alpha^2\lambda\gamma/2$ and
$\tau = \poly(\alpha,\lambda,\gamma)$.
\end{theorem}

That is,
if there is an algorithm for learning the concept class $\C$
over the hypothesis class of real-valued functions
$\H = \set{h:\X \to [-1,1]}$ on the distribution of individuals
in polynomial time in $\log(\card{\C}), 1/\rho,$ and $1/\tau$,
then there is an algorithm for learning
an $\alpha$-multicalibrated predictor on the large sets in $\C$ that runs
in time polynomial in $\log(\card{\C}),1/\alpha,1/\lambda,1/\gamma$.
For clarity of presentation in the reduction,
we make no attempts to
optimize the sample complexity or running time.
Indeed, the exact sample complexity and
running time will largely depend on how strong the weak
learning guarantee is for the specific class $\C$.

We prove Theorem~\ref{thm:weak} by using the weak learner
for $\C$ to learn a $(\alpha,\lambda)$-multicalibrated predictor.
Recall Algorithm~\ref{alg:calibrated}: we maintain
a predictor $x$ and iteratively look for a set
$S \in \C$ where $x$ violates the calibration constraint
on $S_v = \set{i: x_i \in \lambda(v)} \cap S$ for some value $v$.
In fact, the proof of Lemma~\ref{lem:calibrated} reveals
that we are not restricted to updates on $S_v$ for $S \in \C$.
As long as there
is some uncalibrated category $S_v$, we can find an
update that makes nontrivial progress in $\ell_2^2$
distance from $p^*$ -- even if this update is not on
any $S \in \C$ -- then we can bound the number of
iterations it will take before there are no more
uncalibrated categories.  We show that a weak
agnostic learner allows us to find such an update.
\begin{proof}
Throughout the proof, let $\beta = \alpha\lambda\gamma$,
$\rho = \alpha\beta/2$, and
$\tau = \rho^d$ for some constant $d \ge 1$.
Let $x$ be a predictor initialized to be the constant
function $x_i = 1/2$ for all $i \in\X$.

Consider the search problem that arises during
Algorithm~\ref{alg:calibrated} immediately after updating
the predictor $x$.
Let $\X_v = \set{i: x_i \in \lambda(v)}$ be the set
of individuals in the $\lambda$-interval surrounding $v$.
Our goal is to determine if
there is some $v \in \Lambda[0,1]$ and $S \in \C$ such that
$\card{S_v} \ge \beta N$, where
\begin{equation}
\label{eqn:weak}
\card{\sum_{i \in S_v} x_i - p_i^*} \ge \alpha \card{S_v}.
\end{equation}
We reduce this search problem to the problem of
weak agnostic learning over $\C$ on the distribution $\D_\X$.
For any $v \in \Lambda[0,1]$,
if $\card{X_v} < \beta N$, then clearly there is no
uncalibrated category $S_v$ with $\card{S_v} \ge \beta N$;
for each $v \in \Lambda[0,1]$, we will test if $\X_v$ is
large enough by taking $O(\log(1/\beta\xi)/\beta)$
random draws from $\X$.

Supposing $\X_v$ is large enough,
we take a fresh sample of size $n \ge \tilde{O}(\log(\card{\C/\xi})/\beta^2\tau^2)$.
We take $n$ large enough that
over all categories $S_v$ of $\card{S_v} \ge \beta N$,
the observable statistics deviate from their expectation
by at most $\tau/4$:
\begin{equation}
\label{eqn:observable}
\card{\frac{1}{n}\sum_{j\in[n]} o_j -
\frac{1}{\card{S_v}}\sum_{i \in S_v} p_i^*} \le \tau/4
\end{equation}
Additionally, assume that
$x$ is overall observably $\tau/4$-calibrated with respect to $\X$
(recall, this means calibrated on the set of observations).
Note that observable calibration on $\X$ implies that for each
$v \in \Lambda[0,1]$,
\begin{equation}
\card{\frac{1}{n}\sum_{j \in [n]}(o_j - x_j)} \le \tau/4.\label{eqn:boostAE}
\end{equation}
(If $x$ is not $\tau/4$-calibrated
then for the $\X_v$ that violates calibration,
offset all the values of $x_i$ from their
current values such that $\card{\sum_{i \in \X} x_i - o_i}
\le \tau N/4$ and resample; as in Algorithm~\ref{alg:AE}, this
will make at least $\Omega(\tau^2)$ progress towards $p^*$).

For each $v \in \Lambda[0,1]$, we consider the following
learning problem.
For $i \in \X_v$, let $\Delta_i = \frac{x_i - o_i}{2}.$
For $i \in \X \setminus \X_v$, let $\Delta_i = 0$.
We claim that
if there is some $S_v$ satisfying (\ref{eqn:weak}),
then for $i \sim \D_\X$, the labeled samples of
either $(i,\Delta_i)$ or $(i,-\Delta_i)$
satisfy the weak learning promise for
$\rho = \alpha \beta / 2$.
\begin{claim}
Let $c_S:\X \to \set{-1,1}$ be the boolean function
associated with some $S \in \C$.
For $v \in \Lambda[0,1]$, if $S_v = \set{i:x_i \in \lambda(v)} \cap S$
satisfies $\sum_{i \in S_v}(x_i-p_i^*) \ge \alpha \card{S_v}$, then
$$\langle c_S, \Delta \rangle_{\D_\X} \ge \rho.$$
\end{claim}

Note that the supposition of the claim is satisfied when
(\ref{eqn:weak}) holds without the absolute value.
In the case where (\ref{eqn:weak}) holds in the other direction,
the claim will hold for $-\Delta$.  The argument will be identical.
\begin{align}
\langle c_S, \Delta \rangle_{\D_\X} &=
\frac{1}{N}
\sum_{i \in \X}\left(\frac{x_i - o_i}{2}\right)\cdot c_S(i)
\notag\\
&=\frac{1}{N}
\sum_{i \in \X_v}\left(\frac{x_i - o_i}{2}\right)\cdot c_S(i)
+ \sum_{i \in \X \setminus \X_v} 0 \cdot c_S(i)
\notag\\
&=\frac{1}{N}\left(\sum_{i \in S_v}\left(\frac{x_i - o_i}{2}\right)
- \sum_{i \in \X_v \setminus S_v}
\left(\frac{x_i - o_i}{2}\right)\right)
\notag\\
&=\frac{1}{N}\left(\sum_{i \in S_v}\left(\frac{x_i - o_i}{2}\right)
- \left(\sum_{i \in \X_v}\left(\frac{x_i - o_i}{2}\right)
- \sum_{i \in S_v}\left(\frac{x_i - o_i}{2}\right)\right)\right)
\notag\\
&\ge \frac{2}{N}\sum_{i \in S_v}\left(\frac{x_i - o_i}{2}
- \tau \card{\X_v}/4\right)
\label{ineq:one1}\\
&\ge \frac{2}{N}(\alpha\beta N - \tau N/4)
\label{ineq:two2}\\
&\ge 2\rho - \tau / 2
\label{ineq:three3}
\end{align}
where the inequality (\ref{ineq:one1}) follows from (\ref{eqn:boostAE}),
(\ref{ineq:two2}) follows from the assumption that
$\card{S_v} \ge \beta N$ and our assumption on
$\sum_{i \in S_v}(x_i-p_i^*)$, and (\ref{ineq:three3})
follows from our assumption
that the sample size is large enough to guarantee at most
$\tau/8$ error.  Noting that $\tau/2 \le \rho$ gives the claim.

Thus, because the $(\rho,\tau)$-weak agnostic learning
promise is satisfied, the learner will return to us
some $h:\X \to [-1,1]$ satisfying the following inequality.
\begin{align*}
\tau &\le
\langle \Delta_i, h_i \rangle_{\D_v}\\
&= \frac{1}{2\card{\X}}\sum_{i \in \X_v}(x_i - o_i) \cdot h_i\\
&\le \frac{1}{2\card{\X}}\sum_{i \in \X_v}(v - p_i^*) \cdot h_i
+ \tau/8.
\end{align*}
where the final inequality follows by the assumed
statistical accuracy.
This shows that the $h$ returned to us by the
weak agnostic learner is nontrivally correlated
with $x-p^*$ on $\X_v$.
In particular, if we use this $h$
as a gradient step, updating $x_i \to v - \eta h_i$
(projecting onto $[0,1]$ if necessary)
for $\eta = \Omega(\tau/\beta)$, then we can
guarantee that each such update will achieve
$\tau^2N$ progress in $\norm{x-p^*}^2$.
The analysis follows in the same way as the analysis of
Algorithm~\ref{alg:calibrated}.
\end{proof}

\subsection{Weak agnostic learning from multicalibration}

In this section, we show the converse reduction.
In particular, we will show that for a concept class $\C$,
an efficient algorithm for obtaining
an $\alpha$-multicalibrated predictor with respect to
$\C' = \set{S \in \C : \card{S} \ge \gamma N}$,
gives an efficient algorithm for responding to weak agnostic
learning queries on $\C$.

\begin{theorem}[Formal statement of Theorem~\ref{result:learner}]
Let $\alpha,\gamma > 0$ and suppose $\C \subseteq 2^{\X}$
is a concept class.
If there is an algorithm for learning an $\alpha$-multicalibrated
predictor on $\C' = \set{S \in \C : \card{S} \ge \gamma N}$
in time $T(\card{C},\alpha,\gamma)$ then we can implement
a $(\rho,\tau)$-weak agnostic learner for $\C$ in time
$O(T(\card{C},\alpha,\gamma)\cdot \poly(1/\tau))$ for
any $\rho,\tau > 0$ such that $\tau \le \min\set{\rho - 2\gamma,\rho/4 - 4\alpha}$.
\end{theorem}
\begin{proof}
There are two cases to handle.  First, suppose the
support of $c_S$ is small; that is, for the corresponding
$S \in \C$, $\card{S} < \gamma$.  Then, consider the
correlation between $y$ and the the constant hypothesis
$h(i) = -1$ for all $i \in \X$.
\begin{align}
\langle y,-1 \rangle &= -\sum_{i \in \X}y_i\\
&= -\left(\sum_{i \in \X \setminus S}y_i + \sum_{i \in S}y_i\right)\\
&= \langle c_S, y \rangle - 2\sum_{i \in S}y_i\label{eq:weaktrick}\\
&\ge \rho-2\gamma
\end{align}
where (\ref{eq:weaktrick}) follows by writing
$\langle c_S,y \rangle$ as $\sum_{i \in S}y_i + \sum_{i \in \X \setminus S}y_i$
and rearranging.
Thus, for $\tau < \rho - 2\gamma$, in the case when
the support of $c_S$ is small,
then we can return the hypothesis $-1$.
We can test if the constant hypothesis is sufficiently
correlated with $y$ in $\poly(1/\tau)\log(1/\xi)$ time by
random sampling to succeed with probability at least $1-\xi$.

Suppose we want to weak agnostically learn over $\C$
on sampled observations from $y \in [-1,1]^N$.
We assume there is some $c_S \in \C$ such that
$\langle c_S, y \rangle > \rho$.
Consider some $\omega = \rho/4$.
First, we will check if $\frac{1}{N}\card{\sum_{i \in \X}y_i} > \omega$.
Again, this does not require multicalibration, just sampling
from $y$ and averaging.
In this case, a constant function $h(i) = 1$ or
$h(i) = -1$ is sufficiently correlated with $y$
to satisfy the weak agnostic learning guarantee.

Next, we will proceed assuming $\frac{1}{N}\card{\sum_{i \in \X}y_i} < 2\omega$.
We can expand the inner product between $c_S$ and $y$ as
follows.
\begin{align}
\langle c_S, y \rangle &=
\frac{1}{N}\left(\sum_{i \in S}y_i - \sum_{i \in \X \setminus S} y_i\right)\\
&= \frac{1}{N}\left(2\sum_{i \in S}y_i - \sum_{i \in \X}y_i\right)\\
&\ge \frac{2}{N}\sum_{i \in S}y_i - 2\omega
\end{align}
This means $\frac{1}{N}\sum_{i \in S}y_i > \frac{\rho}{2} - \omega$.

Suppose we learn an $x$ that is $\alpha$-multicalibrated with
respect to $\C' = \C \cup \set{\X}$ on the labels $y$.
This implies that there is some $\X'\subseteq \X$ such that
$\card{\X'} \ge (1-\alpha)\card{\X}$ and for all $v \in [-1,1]$, we have
$v - \alpha \le \frac{1}{\card{\X_v'}}\sum_{i \in X_v'} y_i \le v + \alpha$.
In turn, this implies the following inequality.
\begin{equation}
\sgn(v) \cdot \sum_{i \in \X_v'} y_i \ge \card{\X_v'} \cdot \left(\card{v} - \alpha\right)
\label{ineq:abs}
\end{equation}
Then, let $h^{(x)}$ be the hypothesis
defined as $h^{(x)}_i = \sgn(x_i)$.  Consider the inner product with
$y$.
\begin{align}
\langle h^{(x)},y \rangle &= \frac{1}{N}\sum_{i \in \X}h^{(x)}_i \cdot y_i\\
&= \frac{1}{N}\sum_{v \in [-1,1]}\sum_{i \in \X_v}h^{(x)}_i \cdot y_i\\
&\ge \frac{1}{N}\sum_{v \in [-1,1]} \sgn(v)\cdot \sum_{i \in \X_v'} y_i - \alpha
\label{ineq:additional}\\
&\ge \frac{1}{N}\left(\sum_{v \in [-1,1]} \card{\X_v'} \cdot \card{v}\right) - 2\alpha
\label{ineq:sub-v}\\
&\ge \frac{1}{N}\left(\sum_{v \in [-1,1]}
\card{\X_v'} \cdot \card{\frac{1}{\card{S'_v}}\sum_{i \in S'_v} y_i}\right) - 3 \alpha
\label{ineq:sub-Sv}\\
&\ge \frac{1}{N}\left(\sum_{i \in S} y_i\right) - 4\alpha
\label{ineq:sub-S}\\
&\ge \frac{\rho}{2} - \omega - 4\alpha
\end{align}
where the first equalities follow by the definition of $h^{(x)}$;
(\ref{ineq:additional}) follows by the choice of $\X'$ and $\alpha$-multicalibration;
(\ref{ineq:sub-v}) follows by applying (\ref{ineq:abs}) for each $v \in [-1,1]$;
(\ref{ineq:sub-Sv}) follows by substituting $v$ for the empirical average
of $y$ over $S'_v$ invoking $\alpha$-multicalibration for the appropriate choice
of $S' \subseteq S$; and
(\ref{ineq:sub-S}) follows by noting that we can restrict our attention
to $S' \subseteq \X'$ such that $\card{\X'_v} \ge \card{S'_v}$ and
the triangle inequality.

Thus, $h^{(x)}$ satisfies the $(\rho,\tau)$-weak agnostic learning guarantee
for any $\tau \le \rho/4 - 4\alpha$ by our choice of $\omega = \rho/4$.

\end{proof}

\section{Multicalibration achieves ``best-in-class" prediction}
\label{sec:best}

While our notion of multicalibration provides a protection
against discrimination for groups, we argue that this protection comes
at virtually no cost in the utility of the predictor.  In fact,
we argue that Algorithm~\ref{alg:calibrated} can be used as an effective
post-processing step to turn any predictor, or family of
predictors, into a multicalibrated predictor that achieves
comparable (or improved) prediction error.

Suppose we are given a collection $\C$ of sets of individuals
on which we wish to be multicalibrated.  Additionally,
suppose we have a collection of candidate predictors $\H$, which
achieves low prediction error but might violate calibration
arbitrarily. From these collections,
we would like to produce a predictor $x$ that
is $\alpha$-multicalibrated on $\C$ but achieves prediction error
commensurate with the best predictor in $\H$; in particular,
$\norm{x - p^*}^2$ should be not much larger
than $\norm{h^* - p^*}^2$ (and ideally would be smaller).
In this sense, the calibrated $x$ would not only be fair, but
would also achieve (approximately) best-in-class prediction
error over $\H$.

Consider some $h \in \H$ and consider the partition of $\X$
into sets according to the predictions of $h$ -- in particular,
we will first apply a $\lambda$-discretization to the range of
each $h$ to partition $\X$ into categories.
That is, let $S_v(h) = \set{i : h_i \in \lambda(v)}$, and note that
$S_v(h)$ is disjoint from $S_{v'}(h)$ for $v \neq v'$, and $\bigcup_{v \in \Lambda[0,1]} S_v(h) = \X$.  In addition to calibrating
with respect to $S \in \C$,
we can also ask for calibration on $S_v(h)$ for all $h \in \H$
and $v \in \Lambda[0,1]$.  Specifically, let
$\S(\H) = \set{S_v(h)}_{h \in \H, v \in \Lambda[0,1]}$; we
consider imposing calibration on $\C \cup \S(\H)$.
Calibrating in this manner protects the
groups defined by $\C$ but additionally gives a strong utility
guarantee.
\begin{theorem}[Best-in-class prediction]
\label{thm:best}
Suppose $\C \subseteq 2^{\X}$ is a collection of subsets of
$\X$ and $\H$ is a set of predictors.
Then there is a predictor $x$ that is $\alpha$-multicalibrated
on $\C$ such that
$$\norm{x-p^*}^2 - \norm{h^* - p^*}^2 < 6\alpha N,$$
where $h^* = \argmin_{h \in \H}\norm{h - p^*}^2$.
Further, suppose that for all $S \in \C$, $\card{S} \ge \gamma N$, and suppose
that set membership for $S \in \C$ and $h \in \H$
are computable by circuits of size at most $s$; then $x$ is computable
by a circuit of size at most $O(s/\alpha^4\gamma)$.
\end{theorem}
The proof of Theorem~\ref{thm:best}
actually reveals something stronger:  if $x$ is calibrated
on the set $\S(\H)$, then for every category $S_v(h) \in \S(\H)$,
if $x$ is significantly different from $h$ on this category
-- that is, if $\sum_{i \in S_v(h)} (h_i-x_i)^2$ is large --
then $x$ actually achieves significantly improved prediction
error on this category compared to $h$.  This is stated formally
in Lemma~\ref{lem:best}.

\begin{lemma}
\label{lem:best}
Suppose $y$ is an arbitrary predictor and let
$\S(y) = \set{S_v(y)}_{v \in \Lambda[0,1]}$.
Suppose $x$ is an arbitrary $\alpha$-multicalibrated
predictor on $\S(y)$. Then for all $v \in \Lambda[0,1]$,
$$ \sum_{i \in S_v(y)} \left((y_i-p_i^*)^2
- (x_i - p_i^*)^2\right)
\ge 
\sum_{i \in S_v(y)}(v-x_i)^2 - (4 \alpha + \lambda) \card{S_v(y)}.$$
Consequently,
$$\norm{y-p^*}^2 - \norm{x-p^*}^2 \ge \norm{x-y}^2 - (4\alpha +\lambda) N.$$
\end{lemma}
This lemma shows that calibrating on the categories of a predictor
not only prevents the squared prediction error from degrading
beyond a small additive approximation,
but it also guarantees that if calibrating
changes the predictor significantly on any category,
this change represents significant progress towards the true
underlying probabilities on this category.
Assuming Lemma~\ref{lem:best},
Theorem~\ref{thm:best} follows.

\begin{proofof}{Theorem~\ref{thm:best}}
Note that if $x$ is $\alpha$-multicalibrated on $\C$, then $x$ is
$\alpha$-multicalibrated on any $\C' \subseteq \C$.
Consider enforcing calibration on the collection $\C \cup
\S(\H)$ as defined above.
If $x$ is $\alpha$-calibrated on $\C \cup \S(\H)$ then it
is $\alpha$-multicalibrated on $\set{S_v(h)}_{v\in\Lambda[0,1]}$ for all
$h \in \H$ and specifically for $h^*$.
By Lemma~\ref{lem:best}, and the fact that the squared difference
is nonnegative, we obtain the following inequality:
$$ \norm{h^*-p^*}^2 - \norm{x-p^*}^2 \ge
\norm{x-h^*}^2 - (4\alpha + \lambda) N \ge -(4\alpha + \lambda) N.$$
This inequality suffices to prove the accuracy guarantee;
however, to also guarantee the predictor $x$ can be implemented
by a small circuit, we have to be a bit more careful.
In particular, when calibrating, we will ignore any
$S_v(h)$ such that $\card{S_v(h)} < \lambda\alpha N$.
Note that because we have $\lambda$-discretized, there are
at most $1/\lambda$ categories; thus, excluding
the sets $S_v(h)$ where $\card{S_v(h)} < \alpha \lambda N$
introduces at most an additional $\alpha N$ error.
Taking $\lambda = \alpha$,
in turn, this implies that the difference in squared prediction error
can be bounded as $\norm{x-p^*}^2 - \norm{h^*-p^*}^2 \le 6\alpha N$.
Finally, because the sets we want to calibrate on are at least
$\alpha^2\gamma N$ in cardinality,
the circuit complexity bound follows by applying
Algorithm~\ref{alg:calibrated} and Theorem~\ref{thm:ckt}.
\end{proofof}

Thus, given any method
for learning an accurate predictor $h$, we can turn it into
a method for learning a fair and accurate predictor $h'$ by
running Algorithm~\ref{alg:calibrated} on the set of
categories of $h$.
Combined with Theorem~\ref{thm:ckt}, this theorem shows that
for any such class of predictors $\H$ of bounded complexity,
there exists a calibrated predictor with similar circuit
complexity that performs nearly as well as the best
$h \in \H$ in terms of accuracy.
Further, by Lemma~\ref{lem:best}, this (nearly) best-in-class
property will hold not just over the entire domain $\X$, but
on every sufficiently large
category $S_v(h)$ identified by some $h \in \H$.
That is, if $x$ is calibrated on $\S(\H)$, then
for every category $S_v(h)$, the average squared prediction
error $\E_{i \in S_v(h)}\left[(x_i-p_i^*)^2\right]$
will be at most $6\alpha$ worse than prediction given by $h$
on this set. If we view $\H$ as defining a set $\S(\H)$ of
``computationally-identifiable'' categories, then we can view
any predictor that is calibrated on $\S(\H)$ as at least
as fair and at least as accurate on this set of
computationally-identifiable categories as the predictor
that identified the group (up to some small additive
approximation).

We turn to proving Lemma~\ref{lem:best}.
The lemma follows by expanding the difference in squared
prediction errors and invoking the definition of
$\alpha$-calibration.

\begin{proofof}{Lemma~\ref{lem:best}}
Let $S_{vu}$ represent the
set of individuals $i$ where $y \in \lambda(v)$ and $x$
assigns value $u$.
By the assumption that $x$ is
$\alpha$-calibrated on $\S(y)$, we know for every $S_v(y) \in \S(y)$,
there is some subset $S_v'(y) \subseteq S_v(y)$ such that
$\card{S_v'(y)} \ge (1-\alpha)\card{S_v(y)}$ for which $x$'s predictions are approximately correct.  In particular,
let $S_{vu}' = S_v'(y) \cap S_u(x)$; if $x$ is
$\alpha$-calibrated with respect to $S_v(y)$,
this guarantees that for all values $u \in [0,1]$, we have
\begin{equation}
\card{\sum_{i \in S_{vu}'} p_i^* - u} \le
\alpha \card{S_{vu}'}.\label{eqn:cal}
\end{equation}
Using this fact, and the fact that the remaining
$\alpha$-fraction of $S_v(y)$ can contribute at most
$\alpha\card{S_v(y)}$ to the squared error,
we can express the difference in the
squared errors of $y$ and $x$ on $S_v(y)$:
\begin{align}
\sum_{i \in S_v(y)}(y_i - p_i^*)^2
- \sum_{i \in S_v(y)}(x_i - p_i^*)^2 &=
\sum_{i \in S_v(y)}(v - p_i^* + (y_i - v))^2
- \sum_{i \in S_v(y)}(x_i - p_i^*)^2\notag\\
&=\sum_{i \in S_v(y)}(v - p_i^*)^2
- \sum_{i \in S_v(y)}(x_i - p_i^*)^2
+ 2\sum_{i \in S_v(y)}(v-p_i^*)(y_i-v)\notag\\
&\ge \sum_{i \in S_v(y)} (2(p_i^* - v)(x_i-v) - (x_i-v)^2)
- \lambda \card{S_v(y)}
\label{eqn:distance}.
\end{align}
where (\ref{eqn:distance}) follows by the observation that if
$y_i \in \lambda(v)$, then $\card{y_i - v} \le \lambda/2$
and $\card{v-p_i^*}$ is trivially bounded by $1$.
We bound the sum over $i \in \X$ of the first term:
\begin{align*}
\sum_{i \in S_v(y)} (p_i^* - v)(x_i-v)
&=
\sum_{u \in [0,1]}\sum_{i \in S_{vu}}
(p_i^* - v)(u-v)\\
&=
\sum_{u \in [0,1]}(u-v)\sum_{i \in S_{vu}}
(p_i^* - v)\\
&=
\sum_{u \in [0,1]}(u-v)\sum_{i \in S_{vu}}
(u-v + p_i^* - u)\\
&=
\sum_{u \in [0,1]}\left(\card{S_{vu}}(u-v)^2 +
(u-v)\sum_{i \in S_{vu}}(p_i^* - u)\right).
\end{align*}
At this point, we note that $\card{u-v} \le 1$.  Thus,
we can bound the contribution of the sum over $S_{vu}$ by its
negative absolute value:
\begin{align*}
&\ge
\sum_{u \in [0,1]}\left(\card{S_{vu}}(u-v)^2 -
\card{u-v}\card{\sum_{i \in S_{vu}}(p_i^* - u)}\right)\\
&\ge 
\sum_{u \in [0,1]}\left(\card{S_{vu}}(u-v)^2 -
\card{\sum_{i \in S'_{vu}}(p_i^* - u) +
\sum_{i \in S_{vu} \setminus S'_{vu}}(p_i^* - u)}\right)\\
&\ge 
\sum_{u \in [0,1]}\left(
\card{S_{vu}}(u-v)^2 - \card{\sum_{i \in S_{vu}'}(p_i^* - u)}
- \alpha\card{S_v(y)}\right)\\
&\ge
\sum_{u \in [0,1]}
\card{S_{vu}}(u-v)^2 - 2\alpha\card{S_{v}(y)}\\
&=\sum_{i \in S_v(y)}(v-x_i)^2 - 2 \alpha \card{S_v(y)},
\end{align*}
where we bound the sums over $S_{vu}$ by invoking
$\alpha$-calibration and applying (\ref{eqn:cal}).
Plugging this bound into (\ref{eqn:distance}), we
see that
\begin{align*}
\sum_{i \in S_v(y)}\left((y_i-p_i^*)^2
- (x_i-p_i^*)^2\right)
&\ge 2\left(\sum_{i \in S_v(y)}(v-x_i^*)^2
- 2 \alpha \card{S_v(y)}\right)
- \lambda \card{S_v(y)}
- \sum_{i \in S_v(y)} (v-x_i)^2\\
&= \sum_{i \in S_v(y)}(v-x_i^*)^2 - (4\alpha-\lambda) \card{S_v(y)}.
\end{align*}
Summing over $v \in [0,1]$, we can conclude
$$\norm{y-p^*}^2 - \norm{x-p^*}^2 \ge 
\norm{x-y}^2 - (4\alpha-\lambda) N$$
showing the lemma.
\end{proofof}

\paragraph{Acknowlegments}  The authors would like to thank Cynthia Dwork,
Roy Frostig, Parikshit Gopalan, Moritz Hardt, Aditi Raghunathan, Jacob Steinhardt,
and Greg Valiant for helpful discussions related to this work.

\bibliographystyle{alpha}
\bibliography{refs}

\end{document}